\documentclass{article}

\usepackage[preprint]{neurips_2022}

\usepackage[utf8]{inputenc} %
\usepackage[T1]{fontenc}    %
\usepackage{hyperref}       %
\usepackage{url}            %
\usepackage{booktabs}       %
\usepackage{amsfonts}       %
\usepackage{nicefrac}       %
\usepackage{microtype}      %
\usepackage{xcolor}         %
\usepackage{proof-at-the-end}  %

\title{Federated Robustness Propagation: Sharing Robustness in Heterogeneous Federated Learning}

\usepackage{amsthm}
\usepackage{amssymb}
\usepackage{mathtools}
\usepackage{hyperref}
\usepackage[nameinlink,capitalize]{cleveref}
\hypersetup{colorlinks=true,linkcolor=blue,citecolor=blue,urlcolor=blue,pdfborder={0 0 0}}
\usepackage[normalem]{ulem} %
\usepackage{mathtools}

\usepackage[utf8]{inputenc} %
\usepackage[T1]{fontenc}    %
\usepackage{hyperref}       %
\usepackage{url}            %
\usepackage{booktabs}       %
\usepackage{amsfonts}       %
\usepackage{nicefrac}       %
\usepackage{microtype}      %
\usepackage{color, colortbl}
\definecolor{Gray}{gray}{0.9}
\definecolor{lightred}{rgb}{0.909, 0.368, 0.368}
\definecolor{lightblue}{rgb}{0.274, 0.580, 0.866}
\newcolumntype{g}{>{\columncolor{Gray}}c}

\usepackage{graphicx,wrapfig}
\usepackage{algorithm}
\usepackage[noend]{algpseudocode}
\usepackage{caption}
\usepackage{subcaption}
\usepackage{amsfonts}
\usepackage{url}
\usepackage{enumitem}
\usepackage{amsthm}
\usepackage{thmtools}
\usepackage{thm-restate}
\newtheorem{theorem}{Theorem}[section]

\newtheorem{lemma}{Lemma}[section]
\newtheorem{assumption}{Assumption}[section]

\theoremstyle{definition}
\newtheorem{definition}{Definition}[section]
\theoremstyle{remark}

\newcommand{\preprintrm}[1]{}

\newcommand{\cD}{{\mathcal{D}}}

\newcommand{\cH}{{\mathcal{H}}}

\newcommand{\cS}{{\mathcal{S}}}

\newcommand{\cX}{{\mathcal{X}}}

\newcommand{\RR}{\mathbb{R}}

\newcommand{\bc}{\begin{center}}
\newcommand{\ec}{\end{center}}

\newcommand{\bdm}{\begin{displaymath}}
\newcommand{\edm}{\end{displaymath}}

\newcommand{\beq}{\begin{equation}}
\newcommand{\eeq}{\end{equation}}

\newcommand{\bfl}{\begin{flushleft}}
\newcommand{\efl}{\end{flushleft}}

\newcommand{\bt}{\begin{tabbing}}
\newcommand{\et}{\end{tabbing}}

\newcommand{\beqn}{\begin{align}}
\newcommand{\eeqn}{\end{align}}

\newcommand{\beqs}{\begin{align*}} %
\newcommand{\eeqs}{\end{align*}}  %

\newcommand{\norm}[1]{\left\|#1\right\|}

\newcommand{\etal}{\emph{et al.}}
\newcommand{\Ebb}{\mathbb{E}}

\author{%
	Junyuan Hong\textsuperscript{*}
	\And
	Haotao Wang\textsuperscript{\dag}
	\And
	Zhangyang Wang\textsuperscript{\dag}
	\And
	Jiayu Zhou\textsuperscript{*} \\
	\textsuperscript{*}Department of Computer Science and Engineering \\
	Michigan State University \\
	\texttt{\{hongju12,jiayuz\}@msu.edu} \\
	\textsuperscript{\dag}Department of Electrical and Computer Engineering \\
	University of Texas at Austin \\
	\texttt{\{htwang,atlaswang\}@utexas.edu} \\
}

\begin{document}

\maketitle

\begin{abstract}
  Federated learning (FL) emerges as a popular distributed learning schema that learns a model from a set of participating users without sharing raw data. One major challenge of FL comes with heterogeneous users, who may have distributionally different (or non-iid) data and varying computation resources. As federated users would use the model for prediction, they often demand the trained model to be robust against malicious attackers at test time. Whereas adversarial training (AT) provides a sound solution for centralized learning, extending its usage for federated users has imposed significant challenges, as many users may have very limited training data and tight computational budgets, to afford the data-hungry and costly AT. In this paper, we study a novel FL strategy: propagating adversarial robustness from rich-resource users that can afford AT, to those with poor resources that cannot afford it, during federated learning. We show that existing FL techniques cannot be effectively integrated with the strategy to propagate robustness among non-iid users and propose an efficient propagation approach by the proper use of batch-normalization. We demonstrate the rationality and effectiveness of our method through extensive experiments. Especially, the proposed method is shown to grant federated models remarkable robustness even when only a small portion of users afford AT during learning. Source code will be released.
\end{abstract}

\section{Introduction}
\label{sec:intro}

Federated learning (FL)~\citep{mcmahan2017communicationefficient} is a learning paradigm that trains models from distributed users or participants (e.g., mobile devices) without requiring raw training data to be shared, alleviating the rising concern of privacy issues when learning with sensitive data and facilitating learning deep models by enlarging the amount of data for training. 
In a typical FL algorithm, each user trains a model locally using their own data and a server iteratively aggregates users' intermediate models, converging to a model that fuses training information from all users. 

A major challenge in FL comes from two types of the user heterogeneity.
One type of heterogeneity is distributional differences in training data collected by users from diverse user groups, namely \emph{data heterogeneity}~\citep{fallah2020personalized}. 
The heterogeneity should be carefully handled during the learning as a single model trained by FL may fail to accommodate the differences and sacrifices model accuracy~\citep{yu2020salvaging}.
Another type of heterogeneity is the difference of computing resources, named hardware heterogeneity, as different types of hardware used by users usually result in varying computation budgets.
For example, consider an application scenario of FL from mobile phones~\citep{hard2019federated}, where different types of mobile phones (e.g., generations of the same brand) may have drastically different computational power (e.g., memory or CPU frequency). 
As the model size scales with task complexities, the ubiquitous hardware heterogeneity may expel a great number of resource-limited users from the FL process, reduces training data and therefore calls for hardware-aware alternatives~\cite{diao2021heterofl}.

The negative impacts of the heterogeneity become aggravated when an adversarially robust model is desired but its training is not affordable by some users.
The essence of robustness comes from the unnatural vulnerability of models against visually imperceptible noise that can significantly mislead model predictions.
To gain robustness, a straightforward extension of FL, federated adversarial training (FAT), can be adopted~\cite{zizzo2020fat,reisizadeh2020robust}, where each user trains models with adversarially noised samples, namely adversarial training (AT)~\citep{madry2018deep}.
Despite the robustness benefit by AT, prior studies pointed out that the AT is data-thirsty and computationally expensive \citep{shafahi2019adversarial}.
Given the fact that each individual user may not have enough data to perform AT, involving a fair amount of users in FAT becomes essential, but may also induce higher data heterogeneity from diverse data sources.
Meanwhile, the increasingly intensive computation can be prohibitive especially for resource-limited users, that could be $3-10$ times more costly than the standard equivalent \citep{shafahi2019adversarial,zhang2019you}.
As such, it is often unrealistic to enforce \emph{all} users in a FL process to conduct AT locally, despite the fact that the robustness is indeed a strongly desired or even required property for all users. 
This conflict raises a challenging yet interesting question: Is it possible to \emph{propagate adversarial robustness in FL} so that resource-limited users can efficiently benefit from robust training of resource-sufficient users even if the latter has distributionally different data?

Motivated by the question above, we formulate a \emph{novel learning problem} called Federated Robustness Propagation (FRP).
We consider a rather common non-iid FL setting that involves budget-sufficient users (AT users) that conduct adversarial training, and budget-limited ones (ST users) that can only afford standard training.
The goal of FRP is to propagate the adversarial robustness from AT users to ST users, especially when they have different data distributions. 
In \cref{fig:highlight}, we show that independent AT by users without FL (\texttt{local AT}) will not yield a robust model since each user has scarce training data.
Directly extending an existing FL algorithm, e.g., \emph{FedAvg} \citep{mcmahan2017communicationefficient} or a heterogeneity-mitigated one \emph{FedBN} \citep{li2020fedbn} with AT treatments, dubbed FATAvg and FATBN, give very limited capability of robustness.
\begin{wrapfigure}{r}{0.4\textwidth}
    \vskip -0.1in
    \centering
    \includegraphics[width=0.32\textwidth]{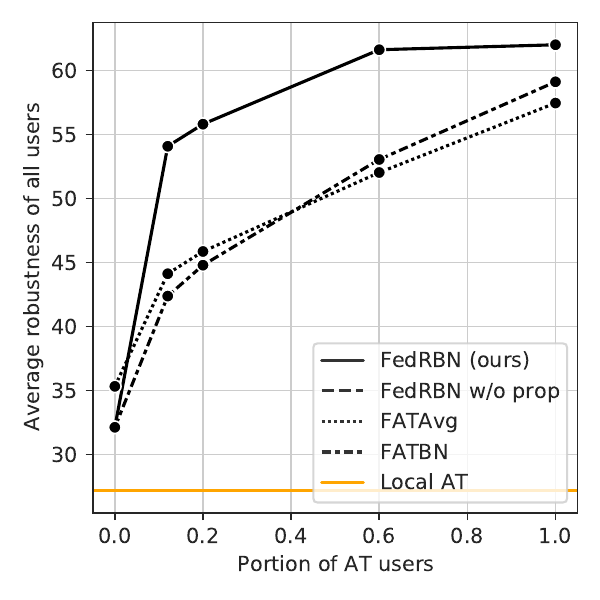}
    \vskip -0.1in
    \caption{
    Comparison of robustness on a varying portion of AT users, where a 5-domain digit recognition dataset is distributed to $50$ users in total and details are in \cref{sec:app:exp_highlight}. %
    }
    \label{fig:highlight}
\end{wrapfigure}

To address the aforementioned challenges, we first provide a \emph{novel insight} that the failure of the traditional method comes from the non-transferable knowledge in the robust BNs.
Even if ST users can borrow the BN parameters from other AT users and make the rest parameters co-trained by all users, the integrated model will not be robust on the ST users.
As conducting AT is so inefficient for ST users, we propose a \emph{novel method} Federated Robust Batch-Normalization (FedRBN) to facilitate efficient sharing of adversarial robustness among users with non-iid data.
With a multi-BN architecture, we propagate adversarial robustness by aggregating the desired knowledge adaptively from multiple AT users to ST users efficiently embedded in few personalized (BN) parameters.
To promote the transferability of robust BNs, we calibrate non-personalized parameters when preserving the robustness of shared noise-aware BNs.
We conduct extensive experiments demonstrating the feasibility and effectiveness of the proposed method.
In \cref{fig:highlight}, we highlight some experimental results from \cref{sec:exp}.
When only $20\%$ of non-iid users used AT during learning, the proposed FedRBN yields robustness, competitive with the best all-AT-user result by only a $6\%$ drop (out of $62\%$) on robust accuracy. 
Note that even if our method with $100\%$ AT users increase the upper bound of robustness, such a bound is usually not attainable in the presence of resource-limited users that cannot afford AT during learning. 

\section{Related Work}
\label{sec:related}
\vskip -0.1in

\textbf{Federated learning for robust models}.
The importance of adversarial robustness in the context of federated learning, i.e., federated adversarial training (FAT), has been discussed in a series of recent literature~\citep{zizzo2020fat,reisizadeh2020robust,kerkouche2020federated}. 
Zizzo~\etal~\citep{zizzo2020fat} empirically evaluated the feasibility of practical FAT configurations (e.g., ratio of adversarial samples) augmenting FedAvg with AT but only in \emph{iid} and label-wise non-\emph{iid} scenarios.
The adversarial attack in FAT was extended to a more general affine form, together with theoretical guarantees of distributional robustness \citep{reisizadeh2020robust,zhang2022distributed,chen2022calfat}.
It was found that in a communication-constrained setting, a significant drop exists both in standard and robust accuracies, especially with non-\emph{iid} data \citep{shah2021adversarial}.
In addition to the challenges investigated above, this work studies challenges imposed by hardware heterogeneity in FL, which was rarely discussed. 
Especially, when only limited users have devices that afford AT, we strive to efficiently share robustness among users, so that users without AT capabilities can also benefit from such robustness.

\textbf{Robust federated optimization}.
Another line of related work focuses on the robust aggregation of federated user updates~\citep{kerkouche2020federated,fu2019attackresistant}.
Especially, Byzantine-robust federated learning~\citep{blanchard2017machine} aims to defend malicious users whose goal is to compromise training, e.g., by model poisoning~\citep{bhagoji2018analyzing,fang2020local} or inserting model backdoor~\citep{bagdasaryan2018how}.
Various strategies aim to eliminate the malicious user updates during federated aggregation~\citep{chen2017distributed,blanchard2017machine,yin2018byzantinerobust,pillutla2020robust}.
However, most of them assume the normal users are from similar distributions with enough samples such that the malicious updates can be detected as outliers.
Therefore, these strategies could be less effective on attacker detection given a finite dataset~\citep{wu2020federated}.
Even though both the proposed FRP and Byzantine-robust studies work with robustness, they have fundamental differences: the proposed work focus on \emph{the robustness during inference}, i.e., after the model is learned and deployed, whereas Byzantine-robust work focus on the robust learning process. 
As such, the proposed approach can combine with all Byzantine-robust techniques to provide training robustness.

\section{Background and Federated Robustness Propagation (FRP)}

In this section, we will review AT, present the unique challenges from hardware heterogeneity in FL and formulate the problem of federated robustness propagation (FRP). In this paper, we assume that a dataset $D$ includes sampled pairs of images $x\in \RR^d$ and labels $y\in \RR^c$ from a distribution $\cD$.
Though our discussion limits the data as images in this paper, the method can be easily generalized to other data forms.
We model a classifier, mapping from the $\RR^{d}$ input space to classification logits $f: \RR^{d} \rightarrow \RR^{c}$, by a deep neural network (DNN) with batch-normalization (BN) layers.
Generally, we split the parameters of $f$ into two parts: $(\mu, \sigma^2)$ including all mean and variance in all BN layers and $\theta$ in others.
To specify a BN structure, e.g., $BN_c$ with identity name $c$ in multiple candidates, we use the notation $f(x; BN_c)$.
Whenever not causing confusion, we use the symbol of a model and its parameters interchangeably.
For brevity, we slightly abuse $\Ebb[\cdot]$ for both empirical average and expectation and use $[N]$ to denote $\{1,\dots,N\}$.

\subsection{Standard training and adversarial training}

An \emph{adversarial attack} applies a bounded noise $\delta_\epsilon: \norm{\delta_\epsilon} \le \epsilon$ to an image $x$ such that the perturbed image $A_\epsilon(x) \triangleq x+\delta_\epsilon$ can mislead a well-trained model to give a wrong prediction.
The norm $\norm{\cdot}$ can take a variety of forms, e.g., $L_{\infty}$-norm for constraining the maximal pixel scale.
A model $f$ is said to be \emph{adversarially robust} if it can predict labels correctly on a perturbed dataset $\tilde D = \{(A_\epsilon(x), y) | (x, y)\in D\}$, and the standard accuracy on $D$ should not be greatly impacted.

Consider the general learning objective: $\min \nolimits_f L(f, D) = \Ebb_{(x,y)\in D} [ \ell(f; x,y)]$.
A user performs \emph{standard training (ST)} if $\ell=\ell_c$ is a standard classification loss on clean images, for example, cross-entropy loss $\ell_{\text{CE}}(f(x), y) = - \sum_{t=1}^c y_t \log (f(x)_t)$ where $t$ is the class index and $f(x)_t$ represents the $t$-th output logit.
In contrast, a user performs \emph{adversarial training (AT)} if $\ell=(\ell_a + \ell_{\text{CE}})/2$ where $\ell_a$ is an adversarial classification loss on noised images.
A popular instantiation of $\ell_a$ is based on PGD attacks \citep{madry2018deep,tsipras2019robustness}:
    $\ell_a(f;x,y) = \max_{\norm{\delta} \le \epsilon} \ell(f(x+\delta), y)$,
where $\norm{\cdot}$ is the $L_\infty$-norm.
With $\ell_c$ and $\ell_a$, we can accordingly define $L_{ST}$ and $L_{AT}$.

\subsection{Learning setup and challenges}

We start with a typical FL setting: a finite set of non-identical distributions $\cD_i$ for $i\in [C]$, from which a set of datasets $\{D_k\}_{k=1}^K$ are sampled and distributed to $K$ users' devices.
The users from distinct domains related with $\cD_i$ expect to learn together while optimize different objectives due to resource constraints: Some users can afford AT training (\emph{AT users} from group $S$) whereas the remaining users cannot afford and use standard training (\emph{ST users} from group $T$). 
The goal of \emph{federated robustness propagation (FRP)} is to transfer the robustness from AT users to ST users at minimal computation and communication costs while preserve data locally. Formally, the FRP objective minimizes:
\begin{align}
&\quad \operatorname{FRP}(\{f_k\}; \{D_k | D_k \sim \cD_i\}) \triangleq \textstyle\sum_{k\in T} L_{\text{ST}}(f_k, D_k) + \textstyle\sum_{k\in S} L_{\text{AT}}(f_k, D_k).  \label{eq:FRP}
\end{align}
In the federated setting, each user's model is trained separately when initialized by a global model, and is aggregated to a global model at the end of each epoch.
A popular aggregation technique is FedAvg~\citep{mcmahan2017communicationefficient},
which averages parameters by $f= {1\over K} \sum_{k=1}^K a_k f_k$ with normalization coefficients $a_k$ proportional to $|D_k|$.
The most related setting to our work is FAT \citep{zizzo2020fat}.
But different from FAT, FRP defined in \cref{eq:FRP} formalizes two types of user heterogeneity that commonly exist in FL.
The first one is the \emph{hardware heterogeneity} where users are divided into two groups by computation budgets ($S$ and $T$).
Besides, \emph{data heterogeneity} is represented as $\cD_i$ differing by domain $i$. We limit our discussion as the common feature distribution shift (on $x$) in contrast to the label distribution shift (on $y$), as previously considered in \citep{li2020fedbn}.
\noindent\textbf{New Challenges.}
 We emphasize that \emph{jointly} addressing the two types of heterogeneity in \cref{eq:FRP} forms a new challenge, distinct from either of them considered exclusively.
First, the scarcity of the AT group worsens the data heterogeneity for additional distribution shift in the hidden representations from adversarial noise~\citep{xie2019intriguing}.
That means even if two users are sampled from the same distribution, their classification layers may operate on different distributions.

Second, the data heterogeneity makes the transfer of robustness non-trivial~\citep{shafahi2019adversarially}.
Hendrycks \etal~\cite{hendrycks2019using} discussed the transfer of models adversarially trained on multiple domains and massive samples.
Later, Shafahi~\etal~\cite{shafahi2019adversarially} firstly studied the transferability of adversarial robustness from one data domain to another by fine-tuning.  %
Distinguished from all existing work, the FRP problem focuses on propagating robustness from multiple AT users to multiple ST users who have diverse distributions and participate in the same federated learning.
Thus, fine-tuning all source models in ST users is often not possible due to prohibitive computation costs.

\section{Method: Federated Robust Batch-Normalization (FedRBN)}

To address the challenges in FRP, we propose a novel federated learning method that propagates robustness using batch-normalization (BN). Recall that BN mitigates the layer distributional shifts and greatly stabilizes the training of very deep networks \cite{ioffe2015batch}.
A BN layer maps a biased variable to a normalized one by
\begin{align}
    \text{BN}(x; \mu, \sigma) \triangleq w \tfrac{x - \mu}{\sqrt{\sigma^2+\epsilon_0}} + b, \label{eq:BN}
\end{align}
where $\mu$ and $\sigma^2$ are the estimated mean and variance over all non-channel dimensions, and $\epsilon_0$ is a small value to avoid zero division.
Since $w$ and $b$ are not distribution-dependent but trainable parameters, we omit them from the notation $\text{BN}(x; \mu, \sigma)$, for brevity.

\subsection{The role of BN revisited}

It is known that batch-normalization can model the internal distributions of activations and mitigate the distribution shifts by normalization.
Therefore, it has been applied to cases where data distribution shifts occur.
The basic principle is to apply \emph{different BNs for different distributions}, by which the output of BN will be distributionally aligned.
In this paper, two kinds of distribution biases are of our interest and their corresponding mitigation methods can be unified into the same principle:
\textbf{(1)} \underline{Feature biases and LBN.}
When users collect data from different sources, their data consists of features biased by different environments.
Though locally trained BNs tend to characterize the biases, the differences captured are immediately forgotten by a global averaging, for instance, in FedAvg.
With the insight, FedBN \cite{li2020fedbn} adopts localized batch-normalization (LBN) for each user, which will be eliminated from the global averaging.
Thus, FedBN outputs $K$ models with LBNs: $\{(\theta, \mu_k, \sigma^2_k)\}_{k=1}^K$.
\textbf{(2)} \underline{Adversarial biases and DBN.}
Recently, \cite{xie2019intriguing} showed that adversarial samples are distributionally biased from clean samples especially in the internal activations of DNN, although the biases are almost invisible in the image.
Such biases substantially lower robustness gained from adversarial training.
Thus, Xie \etal~\cite{xie2020adversarial} proposed a dual batch-normalization (DBN) structure which redirects noised and cleans inputs to different BNs during training: $\text{BN}_a(x; \mu_a, \sigma_a^2)$ given a adversarially-noised $x$ and $\text{BN}_c(x; \mu, \sigma^2)$ given clean $x$.
For example, the adversarial training will instead optimize $\ell_c(f(x; \text{BN}_c)) + \ell_a(f(x; \text{BN}_a))$.
After training, it is recommended to use $\text{BN}_a$ for improved robustness.
Though not as accurate as $\text{BN}_c$, $\text{BN}_a$ are still accurate.

\textbf{Joint use of LBN and DBN in FRP.}
Because of the co-occurrence of feature heterogeneity and adversarial training in FRP, it is natural to adopt both LBN and DBN in FL.
We name the combination as FATBN+DBN.
That admits an extended set of BN parameters, $(\mu_k, \sigma^2_k)$ (clean), and $(\mu_{a,k}, \sigma^2_{a,k})$ (adversarial), for user $k$.
Since the essence of LBN and DBN are well established, it should be natural to use them together when two kinds of biases present.
Interested readers may be referred to \cref{sec:app:dist_bias_and_bn} for qualitative evidence of such essence.
Later in benchmark experiments (c.f. \cref{tbl:bmk_single_source_prop}), we also show that the joint use boosts the robustness than using one of them exclusively.

However, the accuracy boosting comes with the challenges for FRP, when ST users cannot afford the AT due to the limited computation resources.
Without globally aggregating DBNs, ST users have to leave one branch of DBN blank or random, because no adversarial samples are provided to tune them.
The missing branch makes the ever-successful method inapplicable with the device heterogeneity.
Thus, an efficient manner without heavy computation overhead is desired to fill the gap.

\subsection{BN-based propagation}

To address the problem, we propose a simple estimation of the missing BN$_a$ with global averaging:
\begin{align}
    \textstyle \hat \mu_{a,k} = {1\over |S|} \sum_{j \in S} \alpha_j \mu_{a,j}, ~\hat \sigma^2_{a,k} = {1\over |S|} \sum_{j \in S} \alpha_j \sigma^2_{a,j}, \label{eq:ABN}
\end{align}
where $\alpha_j$ is a normalized weight.
As \cref{eq:ABN} is simply a linear operation, the estimation is very efficient due to the small portion of BN parameters in a deep network.
To find an ideal $\alpha$ minimizing the adversarial loss during inference, below we theoretically show that the divergence of a clean pair bounds the generalizable adversarial loss, given bounded adversarial bias.

\begin{lemma}[Informal state of \cref{thm:multi_src_domain_formal}] \label{thm:multi_src_domain}
    Suppose the divergence between any data distribution $\cD$ and its adversarial one $\tilde \cD$ is bounded by a constant, i.e., $d_{\cH \Delta \cH} (\tilde \cD, \cD) \le d_{\epsilon}$ where $d_{\cH \Delta \cH}$ is $\cH \Delta \cH$-divergence in hypothesis space $\cH$.
    If a target model is formed by \cref{eq:ABN} of models trained on a set of source standard datasets $\{D_{s_i}\}$, its generalization error on the target $\tilde \cD_t$ is upper bounded by the weighted summation $\sum_i \alpha_i d_{\cH \Delta \cH} (D_{s_i}, D_t)$ of paired standard divergence given $D_t\sim \cD_t$. %
\end{lemma}

The lemma extends an existing bound for federated domain adaptation \citep{peng2019federated}, and shows that the generalization error on the unseen target noised distribution $\tilde D_t$ is bounded by the $\alpha_i$-weighted standard distribution gaps.

\textbf{New client similarity measure for adaptive propagation.}
Results in \cref{thm:multi_src_domain} and the domain gaps between adversarial samples motivate us 
to set $\alpha_j$ to be reversely proportional to the divergence between $D_k$ and $D_j$.
Since other users' data are not available, directly modeling the divergence is by data is prohibitive.
Fortunately, as clean BN statistics characterize each user's data distributions, we can use a layer-averaged similarity to approximate the weight, i.e.,
\begin{align}
    \textstyle \alpha_j = \text{Softmax}_T \bigl[ \tfrac{1}{L} \sum \nolimits_{l=1}^L \text{Sim}^l(D_k, D_j) \bigr],
\end{align}
where $\text{Softmax}_T(q_j)$ is a tempered softmax function: $\exp(q_j/T)/\sum_{j\in S} \exp(q_j/T)$. 
$T$ equals $0.01$ by default in this paper.
The $l$-th-layer similarity is approximated by the BN statistics: $\text{Sim}^l(D_k, D_j) = [\cos(\mu_k^l, \mu_j^l) + \cos({\sigma_k^2}^l, {\sigma_j^2}^l)]/2$ given $\cos(x,y)=x^\top y/\norm{x}\norm{y}$.

\textbf{Non-reducible gap in BN and clean adaption of non-BN parameters.}
\cref{thm:multi_src_domain} suggests that the optimal divergence will be no better than the divergence by the model from the most similar source.
When all source datasets are from domains distinguished from the target domain, then estimated BN parameters by \cref{eq:ABN} cannot further compress divergence and improve adversarial losses.
In \cref{fig:pen_fea_viz_BNa_FRP-full-OtoM}, we show the non-reducible domain gap between MNIST and SVHN: the transferred $\text{BN}_a$ yields a much less discriminative representations than the locally trained $\text{BN}_a$ during training.
\begin{wrapfigure}{r}{0.5\textwidth}
    \vskip -0.1in
        \centering
        \includegraphics[width=0.49\textwidth]{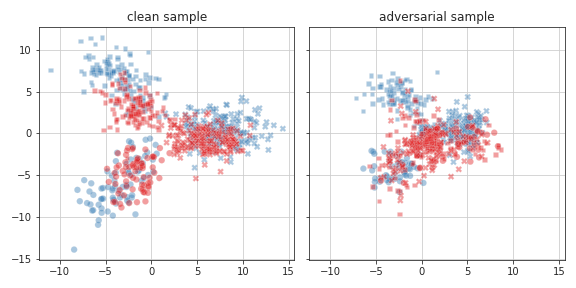}
    \vskip -0.1in
    \caption{Penultimate layer representations visualized by a Digits model and SVHN-domain users. The visualization projects 400 randomly-selected samples into the first three classes of SVHN datasets following \cite{muller2019when}. 
    Representations are computed by \textcolor{lightblue}{trained} or \textcolor{lightred}{transferred} $\text{BN}_a$. The model is trained by \textcolor{lightblue}{100\%-AT} or \textcolor{lightred}{1-domain-AT} user. In the latter setting, $\text{BN}_a$ is propagated according to \cref{eq:ABN}.
    More figures in \cref{fig:pen_fea_viz_ex}.
    }
    \label{fig:pen_fea_viz_BNa_FRP-full-OtoM}
    \vskip -0.1in
\end{wrapfigure}
In addition, we surprisingly observe that the clean discrimination is not well transferred, either.
The observation implies that though non-BN parameters are trained and adapted towards different domains, the single-domain $\text{BN}_a$ still cast biases into the representations even for clean samples.
To fix this, we propose a clean adaptation of the federated model, which calibrates non-BN parameters by local clean features only:
(1) Given the estimated $(\hat \mu_{a,k}, \hat \sigma^2_{a,k})$, keep the two parameters \emph{frozen} to avoid statistic interference from clean samples.
As the distributional biases between domains are typically larger than that between clean and adversarial statistics, freezing $\text{BN}_a$ can impede catastrophic forgetting of the critical robustness knowledge. 
(2) Optimize an augmented ST loss:
\begin{align}
    \textstyle (1-\lambda) \ell_{\text{CE}}\big(f_k(x; \text{BN}_c),y\big) + \lambda \ell_{\text{CE}}\big(f_k(x; \text{BN}_a), y \big), \label{eq:PNC_loss}
\end{align}
where the second term, pseudo-noise calibration (PNC) loss, augments the robustness by $\text{BN}_a$ without computation-intensive adversarial attacks. %
If without the domain gap, $f_k(x; \text{BN}_a)$ will bias the outputs on clean input $x$, which functions like noising the training process.
Otherwise, frozen $\text{BN}_a$ can calibrate the other parameters to mitigate the distributional bias such that the robustness encoded in $\text{BN}_a$ is transferable.
In \cref{eq:PNC_loss}, the hyper-parameter $\lambda$ is set to be $0.5$ by default, which provides a fair trade-off between robustness and accuracy.
A smaller $\lambda$ can be used to trade in robustness for accuracy, or vice versa.

\begin{minipage}{0.5\textwidth}
\begin{algorithm}[H]
	\centering
	\small
	\caption{FedRBN: user-end training}
	\label{alg:user_train}
	\begin{algorithmic}[1]
	\Require User budget type (AT or ST), initial parameters $\theta$ (AT) or $(\theta, \hat \mu_a, \hat \sigma^2_a)$ (ST) of the model $f$ from the server, adversary $A_{\epsilon}(\cdot)$, dataset $D$
    \For{mini-batch $\{(x, y)\}$ in $D$}
        \State $\ell_c \leftarrow \Ebb_{(x,y)} [ \ell_{\text{CE}}(f(x; \text{BN}_c),y)]$
        \State Update $(\mu, \sigma^2)$ of $\text{BN}_c$
        \If{user budget type is AT}
            \State Perturb data $\tilde x \leftarrow A_{\epsilon}(f(x; \text{BN}_a))$
            \State $L \leftarrow {1\over 2} \left\{ \ell_c + \Ebb_{(\tilde x,y)} [ \ell_{\text{CE}}(f(\tilde x; \text{BN}_a),y)] \right\}$
            \State Update $(\mu_a, \sigma^2_a)$ of $\text{BN}_a$
        \Else
            \State \textcolor{blue}{Replace $\text{BN}_a$ parameters with $(\hat \mu_a, \hat \sigma^2_a)$}
            \State \textcolor{blue}{$L \leftarrow (1-\lambda) \ell_c$} \\
            \hspace{0.5in}\textcolor{blue}{$+ \lambda \Ebb_{(x,y)} [ \ell_{\text{CE}}(f(x; \text{BN}_a),y)]$}
        \EndIf
        \State Optimize $L$ to update $\theta$ by gradient descent
    \EndFor
    \State \textbf{Upload} $(\theta, \mu, \sigma^2, \mu_a, \sigma_a^2)$ (AT) or $(\theta, \mu, \sigma^2)$ (ST)
	\end{algorithmic}
\end{algorithm}
\end{minipage}
\hfill
\begin{minipage}{0.49\textwidth}
\begin{algorithm}[H]
	\centering
	\small
	\caption{FedRBN: server-end training}
	\label{alg:svr_train}
	\begin{algorithmic}[1]
	\Require An initial model $f$ with BN parameters $(\hat \mu_{a}, \hat \sigma^2_{a})$ and other non-BN parameters $\theta$, $K$ users belonging to $S$ (AT) or $T$ (ST) sets, total iteration number $\tau$
    \For{$t \in \{1, \dots, \tau\}$}
        \State Send global model $\theta_k$ to users indexed by $k\in S$ and $(\theta_k^t, \hat \mu_{a,k}, \hat \sigma^2_{a,k})$ to users indexed by $k \in T$
        \State In parallel, users train their models by \cref{alg:user_train}
        \State Receive users' parameters: $$\hspace{-0.2in}\{(\theta_k, \mu_k, \sigma^2_k)\}_{k\in T}\text{ and }\{(\theta_k, \mu_k, \sigma^2_k, \mu_{k,a}, \sigma^2_{k,a})\}_{k\in S}$$
        \State Average parameters: $\theta \leftarrow{1\over K} \sum_{k=1}^K \theta_k$
        \State \textcolor{blue}{Use $\{(\mu_k, \sigma^2_k, \mu_{k,a}, \sigma^2_{k,a})\}_{k\in S}$ to estimate adversarial BN parameters $\{(\hat \mu_{k,a}, \hat \sigma^2_{k,a})\}_{k\in T}$ by \cref{eq:ABN}}
    \EndFor
    \State \textbf{Return} $K$ models parameterized by $\{(\theta, \mu_k, \sigma^2_k, \mu_{k,a}, \sigma^2_{k,a})\}_k$
	\end{algorithmic}
\end{algorithm}
\end{minipage}

\subsection{FedRBN algorithm and its efficiency}

We are now ready to present the proposed BN-based FRP algorithm: Federated Robust Batch-Normalization (FedRBN).
On the user side (\cref{alg:user_train}), we introduce a standard loss in addition to the standard federated adversarial training.
The loss is embarrassingly simple and easy to implement in two lines, as highlighted.
On the server side (\cref{alg:svr_train}), we follow the same practice as FedAvg to aggregate models and average (perhaps weighted if users' sample sizes differs).
Different from FedAvg, we drop unnecessary parameter sharing like sending BN parameters to AT users and leverage the globally shared BN parameters to estimate missing $\text{BN}_a$ parameters.

\textbf{Efficiency and privacy of BN operations}.
Since BN statistics are only a tiny portion of any networks and do not require back-propagation, an additional set of BN statistics will marginally impact the efficiency \citep{wang2020onceforall}. %
During training, the communication cost is almost the same as the most popular FL method, FedAvg \citep{mcmahan2017communicationefficient}, with a small portion of additional BN parameters.
On the user side, the major computation overhead comes from the additional loss, which doubles the complexity of a ST user.
However, the overhead is much cheaper than adversarial training, which typically requires multiple iterations (e.g., 7 steps \cite{madry2018deep}) of gradient descent for attacks.
Many existing FL designs such as FedAvg have privacy concerns~\citep{li2020federateda,fallah2020personalized}, and sharing local statistics can also contribute to potential privacy leakage~\citep{geiping2020inverting}. 
Though not the scope of this work, we can implement protection by applying differential privacy mechanism \cite{dwork2006calibrating} on the BN statistic estimation, where a minor Gaussian noise is injected on every statistic update in \cref{alg:user_train}.

\vspace{-0.1in}
\section{Experiments}
\label{sec:exp}
\vspace{-0.1in}

\textbf{Datasets and models}.
To implement a non-iid scenario, we adopt a close-to-reality setting where users' datasets are sampled from different distributions. We used two multi-domain datasets for the setting.
The first is a subset (30\%) of \textsc{Digits}, a benchmark for domain adaption~\citep{peng2019federated}.
\textsc{Digits} has $28\times28$ images and serves as a commonly used benchmark for FL~\citep{caldas2019leaf,mcmahan2017communicationefficient,li2020federateda}. 
\textsc{Digits} includes $5$ different domains: MNIST (MM) \citep{lecun1998gradientbased}, SVHN (SV) \citep{netzer2011reading}, USPS (US) \citep{hull1994database}, SynthDigits (SY) \citep{ganin2015unsupervised}, and MNIST-M (MM) \citep{ganin2015unsupervised}.
The second dataset is \textsc{DomainNet}~\citep{peng2019moment} processed by \citep{li2020fedbn}, which contains 6 distinct domains of large-size $256\times 256$ real-world images: Clipart (C), Infograph (I), Painting (P), Quickdraw (Q), Real (R), Sketch (S).
For \textsc{Digits}, we use a convolutional network with BN (or DBN) layers following each conv or linear layers. 
For the large-sized \textsc{DomainNet}, we use AlexNet \citep{krizhevsky2012imagenet} extended with BN layers after each convolutional or linear layer following prior non-iid FL practice~\citep{li2020fedbn}.
\preprintrm{One more large-sized image dataset is presented in \cref{sec:office_results}.}

\textbf{Training and evaluation}. %
For AT users, we use $n$-step PGD (projected gradient descent) attack \citep{madry2018deep} with a constant noise magnitude $\epsilon$.
Following \citep{madry2018deep}, we use $\epsilon=8/255$, $n=7$, and attack inner-loop step size $2/255$, for training, validation, and test. %
We uniformly split the dataset for each domain into $10$ subsets for \textsc{Digits} and $5$ for \textsc{DomainNet}, following \citep{li2020fedbn}, which are distributed to different users, respectively.
Accordingly, we have $50$ users for \textsc{Digits} and $30$ for \textsc{DomainNet}.
Each user trains local model for one epoch per communication round.
We evaluate the federated performance by \underline{standard accuracy} (SA), classification accuracy on the clean test set, and \underline{robust accuracy} (RA), classification accuracy on adversarial images perturbed from the original test set.
All metric values are averaged over users.
We defer other details of experimental setup such as hyper-parameters to \cref{sec:app:exp}, and focus on discussing the results.

\vspace{-0.5em}
\subsection{Comprehensive study}
\vspace{-0.5em}
To further understand the role of each component in FedRBN, we conduct a comprehensive study on its properties.
In experiments, we use three representative federated baselines combined with AT: FedAvg~\citep{mcmahan2017communicationefficient}, FedProx~\citep{li2020federateda}, and FedBN~\citep{li2020fedbn}.
We use FATAvg to denote the AT-augmented FedAvg, and similarly FATProx and FATBN. To implement hardware heterogeneity, we let $20\%$-per-domain users from $3/5$ domains (of \textsc{Digits}) conduct AT.

\textbf{Ablation Studies}.
We study how BN should be used at inference time when LBN and DBN are already integrated into federated training.
Thus, we evaluate trained models with users' local $\text{BN}_c$ and $\text{BN}_a$ transmitted from the global estimation.
We also compare two kinds of weighting strategy for estimating transferable $\text{BN}_a$ parameters: uniform weights (uni) or the proposed cosine-similarity-based weights for soruce users.
In \cref{tbl:ablation_comp}, we present the results with $\lambda \in \{0,0.5\}$ for PNC losses.
\begin{wraptable}{r}{0.55\textwidth}
  \caption{Ablation of different test-time BNs.}
  \label{tbl:ablation_comp}
  \tiny
  \setlength\tabcolsep{1.5 pt}
  \centering
  \begin{tabular}{*{3}{c}|*{18}{c}}
    \toprule
   $\lambda$ & test BN & weight & \multicolumn{6}{c}{Digits} & \multicolumn{6}{c}{DomainNet} \\
    \cmidrule(r){4-9} \cmidrule(r){10-15}
   & &  & \multicolumn{2}{c}{All} & \multicolumn{2}{c}{20\%} & \multicolumn{2}{c}{MNIST} & \multicolumn{2}{c}{All} & \multicolumn{2}{c}{20\%} & \multicolumn{2}{c}{Real} \\
    \cmidrule(r){4-5} \cmidrule(r){6-7} \cmidrule(r){8-9} 
    \cmidrule(r){10-11} \cmidrule(r){12-13} \cmidrule(r){14-15}
   & &  & RA & SA & RA & SA & RA  & SA & RA & SA & RA & SA & RA & SA \\
    \midrule
     0 & $\text{BN}_c$      & & 52.8 & \textbf{86.7} & 41.9 & 86.6 & 34.6 & 84.7 & 35.5 & 61.4 & \textbf{22.1} & \textbf{65.0} & 15.4 & \textbf{65.9} \\
     0 & tran. $\text{BN}_a$ & uni & \textbf{62.0} & 84.9 & 50.6 & 83.2 & \textbf{41.5} & 80.2 & \textbf{35.7} & \textbf{61.6} & 19.8 & 60.5 & 13.2 & 56.1 \\
     0 & tran. $\text{BN}_a$ & cos & \textbf{62.0} & 84.9  & \textbf{51.0} & 83.5 &  \textbf{41.5} & 80.2 & \textbf{35.7} & \textbf{61.6} & 21.4 & 62.5 & 12.8 & 56.1  \\
    \midrule
     0.5 & $\text{BN}_c$      & & 52.8 & \textbf{86.7}  & 50.0 & 87.0 & 42.2 & 84.1 & 35.5 & 61.4 & 26.5  & 61.2 & 21.0 & 62.0  \\
     0.5 & tran. $\text{BN}_a$ & uni & \textbf{62.0} & 84.9   & 55.4 & 86.9 & 51.5 & 87.2 & \textbf{35.7} & \textbf{61.6} & 27.5 & 61.3 & \textbf{26.4} & \textbf{64.0}  \\
     0.5 & tran. $\text{BN}_a$ & cos & \textbf{62.0} & 84.9   & \textbf{55.8} &          \textbf{87.3}   & \textbf{58.5} &          \textbf{86.5}  & \textbf{35.7} & \textbf{61.6}  &  \textbf{28.1} &     \textbf{62.5}       & \textbf{26.4} &  63.9    \\
    \bottomrule %
  \end{tabular}
\end{wraptable}
When $\lambda=0$, we only share robustness through customizing $\text{BN}_a$ for each target ST user without PNC losses and the propagated BNs is more effective on the Digits than on DomainNet, because DomainNet is a more complicated task involving higher domain divergence.
As the domain gap overwhelms the gap between adversarial samples and clean samples (also see representation comparison in \cref{fig:bn_pca_ex}), the $\text{BN}_c$ outperforms the $\text{BN}_a$ surprisingly on RA.
As we formerly discussed, the non-reducible domain gap in adversarial training motivates our development of PNC loss.
With PNC loss ($\lambda=0.5$), we significantly improves the robustness and accuracy and the performance approaches the all-AT results.
In addition, either with or without PNC losses, the cos-weighting strategy consistently improves the robustness compared to non-informative uniform weights. \\
\textbf{Impacts from data heterogeneity}.
To study the influence of different AT domains, we set up an experiment where AT users only reside on one single domain.
For simplicity, we let each domain contains a single user as in \citep{li2020fedbn} and utilize only 10\% of \textsc{Digits} dataset.
The single AT domain plays the central role in gaining robustness from adversarial augmentation and propagating to other domains.
The task is hardened by the non-singleton of gaps between the AT domain and multiple ST domains and a lack of the knowledge of domain relations.
Results in \cref{fig:benchmark_O2M} show the superiority of the proposed FedRBN, which improves RA for more than $10\%$ in all cases with small drops in SA.
We see that RA is the worst when MNIST serves as the AT domain, whereas RA propagates better when the AT domain is SVHN or SynthDigits.
A possible explanation is that SVHN and SynthDigits are more visually distinct than the rest domains (see \cref{fig:dataset_samples}), forming larger domain gaps.
\newline
\textbf{Impacts from hardware heterogeneity}.
We vary the number of AT users in training from $1/N$ (most heterogeneous) to $N/N$ (homogeneous) to compare the robustness gain.
\cref{fig:benchmark_partial_M2M} shows that our method consistently improves the robustness.
Even when all domains are noised, FedRBN is the best due to the use of DBN.
When not all domains are AT, our method only needs half of the users to be noised such that the RA is close to the upper bound (fully noised case).
\newline
\textbf{Other comprehensive studies in Appendix} for interested readers.
Concretely, we studied the $\lambda$-governed trade-off (\ref{sec:app:trade_off_pnc}), the convergence curves (\ref{sec:conv_hyper}), \preprintrm{more datasets (\ref{sec:office_results}), other models (\ref{sec:app:resnet}),} detailed ablation studies of FL configurations (\ref{sec:app:fl_conf})\preprintrm{, and domain-wise analysis of accuracy (\ref{sec:app:domain})}.

\begin{figure*}[!b]
    \vspace{-0.2in}
    \centering
    \begin{subfigure}{0.49\textwidth}
        \centering
        \includegraphics[width=0.49\textwidth]{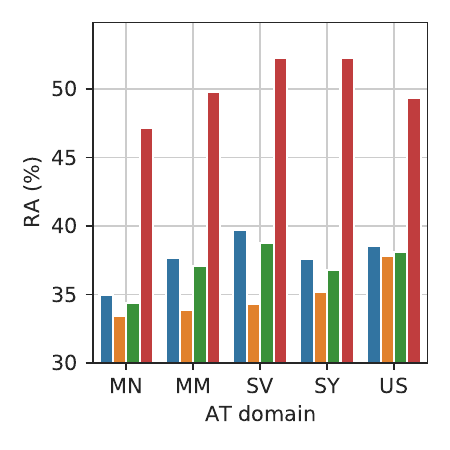}
        \hfil
        \includegraphics[width=0.49\textwidth]{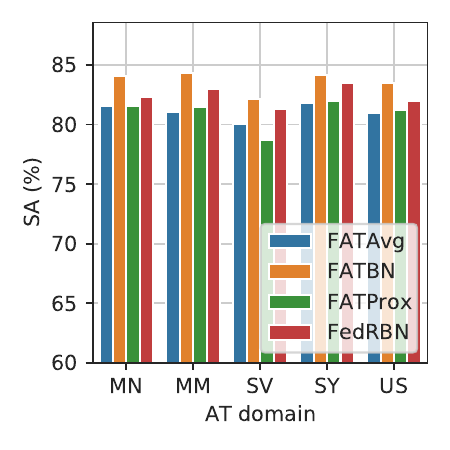}
        \caption{FRP from a single AT domain}
        \label{fig:benchmark_O2M}
    \end{subfigure}
    \begin{subfigure}{0.49\textwidth}
        \centering
        \includegraphics[width=0.49\textwidth]{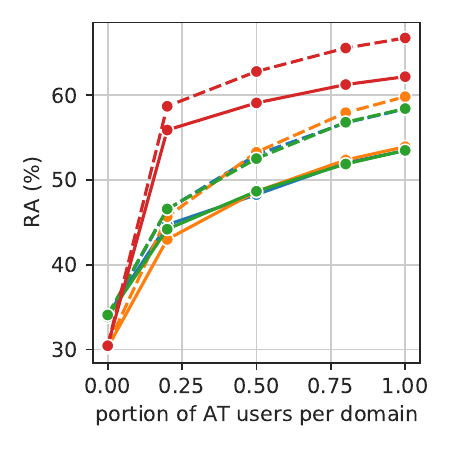}
        \hfil
        \includegraphics[width=0.49\textwidth]{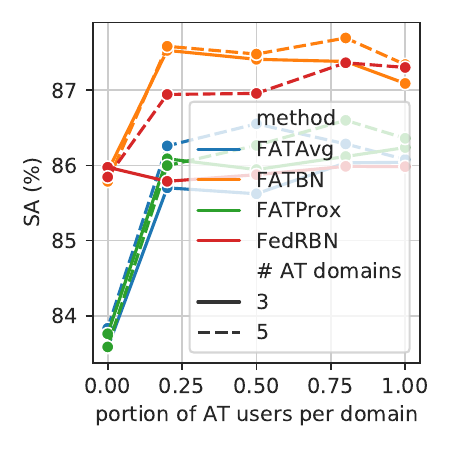}
        \caption{FRP from partial AT users per domain}
        \label{fig:benchmark_partial_M2M}
    \end{subfigure}
    \vskip -0.1in
    \caption{Evaluating FRP performance with different FRP settings.
    }
    \label{fig:frp_varying}
    \vspace{-0.2in}
\end{figure*}

\vspace{-0.5em}
\subsection{Comparison to baselines}
\label{sec:comp_baselines}
\vspace{-0.5em}
To demonstrate the effectiveness of the proposed FedRBN, we compare it with baselines on two benchmarks.
We repeat each experiment for three times with different seeds.
We introduce two more baselines:
a proposed baseline combining FATAvg with DBN,
personalized meta-FL extended with FAT (FATMeta)~\citep{fallah2020personalized} and federated robust training (FedRob)~\citep{reisizadeh2020robust}.
Because FedRob requires a project matrix of the squared size of image and the matrix is up to $256^2\times256^2$ on \textsc{DomainNet} which does not fit into a common GPU, we exclude it from comparison.
Given the same setting, we constrain the computation cost in the similar scale for cost-fair comparison.
We evaluate methods on two FRP settings.
\textbf{1) Propagate from a single domain}.
In reality, a powerful computation center may join the FL with many other users, e.g., mobile devices.
Therefore, the computation center is an ideal node for the computation-intensive AT.
Due to limitations of data collection, the center may only have access to a single domain, resulting gaps to most other users.
We evaluate how well the robustness can be propagated from the center to others.
\textbf{2) Propagate from a few multi-domain AT users}.
In this case, we assume that to reduce the total training time, ST users are exempted from the AT tasks in each domain.
Thus, an ST user wants to gain robustness from other same-domain users,
but the different-domain users may hinder the robustness due to the domain gaps in adversarial samples. \\
\textbf{Benchmark.}
\cref{tbl:bmk_single_source_prop} shows that our method outperforms all baselines for all tasks, while it associates to only small overhead (for optimizing PNC losses) compared to the full-AT case.
Importantly, we show that only $20\%$ users and less than 33\% time complexity of the full-AT setting are enough to achieve robustness comparable to the best fully-trained baseline.
Contradicting FATAvg+DBN and FATBN confirmed the importance of DBN in robustness but also show its limitation on handling data heterogeneity.
Thus, FedRBN ($\lambda=0$) is proposed to simultaneously address data and hardware heterogeneity by efficiently propagating robustness through BNs. 
To fully exploit the robustness complying users' hardware limitations, the PNC loss ($\lambda>0$) is used and improves the robustness significantly.
When $\lambda=1$, the trained inclines to be more robust but less accurate on clean samples.
Instead, $\lambda=0.5$ provides a fairly nice trade-off between accuracy and robustness, for which we use the parameter generally. \\
\textbf{Stronger attacks.} To fully evaluate the robustness, we experiment with more attack methods, including MIA \citep{dong2018boosting}, AutoAttack (AA) \citep{croce2020reliable} and LSA \citep{narodytska2016simple}.
A strong score-based blackbox attacks such as Square Attack \citep{andriushchenko2020square} (included in AA) can avoid the trip fake robustness due to obfuscated gradient 
Even evaluated by different attacks (see \cref{tbl:eval_attacks}), our method still outperforms others.
\\
\textbf{Compare to full efficient AT}.
In \cref{tbl:bmk_single_source_prop_free}, we show that when computation time is comparable, our method can achieve both better RA and SA than full-AT baselines.
For results to be comparable, we train FedRBN for limited 150 epochs while Free AT for 300 epochs.
Although Free AT improves the robustness compared to FATAvg, it also greatly sacrifices SA performance.
Thanks to stable convergence and decoupled BN, FedRBN maintains both accurate and robust performance though the AT is not `free' for a few users.

\begin{table*}[t!]
  \vskip -0.1in
  \caption{Benchmarks of robustness propagation, where we measure the per-epoch computation time ($T$) by counting $\times 10^{12}$ times of multiplication-or-add operations (MACs) to evaluate the \textbf{efficiency}.}
  \label{tbl:bmk_single_source_prop}
  \scriptsize
  \setlength\tabcolsep{1.5 pt}
  \centering
  \begin{tabular}{l*{2}{c}*{18}{c}}
    \toprule
    & LBN & DBN & \multicolumn{9}{c}{Digits} & \multicolumn{9}{c}{DomainNet} \\
    \cmidrule(r){4-12} \cmidrule(r){13-21}
  AT users & &  & \multicolumn{3}{c}{All} & \multicolumn{3}{c}{20\%} & \multicolumn{3}{c}{MNIST} & \multicolumn{3}{c}{All} & \multicolumn{3}{c}{20\%} & \multicolumn{3}{c}{Real} \\
    \cmidrule(r){4-6} \cmidrule(r){7-9} \cmidrule(r){10-12} 
    \cmidrule(r){13-15} \cmidrule(r){16-18} \cmidrule(r){19-21}
  Metrics & &  & RA & SA & T & RA & SA & T & RA  & SA & T & RA & SA & T & RA & SA & T & RA & SA & T \\
  \midrule
  FedRBN $\lambda=1$  & \checkmark & \checkmark & \textbf{62.0} & 84.9 & 7.4   & \textbf{60.6} &          86.5 &    2.5     & \textbf{60.8} &         83.9  &       2.5  & \textbf{35.7} & \textbf{61.6} &  127.9    & 27.6 &     56.0    &    42.6   & \textbf{28.2} &      58.3    &      39.1  \\ %
  FedRBN $\lambda=0.5$  & \checkmark & \checkmark  & \textbf{62.0} & 84.9 & 7.4    & 55.8 &          \textbf{87.3} &     2.9    & 58.5 &          \textbf{86.5} &       2.9  & \textbf{35.7} & \textbf{61.6} & 127.9  & \textbf{28.1} &     62.5   &  51.2     & 26.4 &  63.9  &    48.0   \\
  FedRBN $\lambda=0$  & \checkmark & \checkmark & \textbf{62.0} & 84.9 & 7.4 & 51.0 & 83.5 & 2.2 &  41.5 & 80.2 & 2.2 & \textbf{35.7} & \textbf{61.6} & 127.9 & 21.4 & 62.5 & 38.4 & 12.8 & 56.1 & 34.6  \\
  \midrule
  FATAvg+DBN  &  & \checkmark & 60.0 & 83.8 & 7.4 & 48.8 & 82.8 & 2.2 & 40.2 & 79.9 & 2.2 & 27.6 & 52.8 & 127.9 & 16.6 & 58.9 & 38.4 & 13.0 & 54.8 & 34.6  \\
  FATBN  & \checkmark &   &          60.0 & \textbf{87.3} &       7.4 &          41.2 & 86.4 &        2.2 &          36.5 & 86.4 &        2.2 &          35.2 &          60.2 &      127.9 &          20.3 & \textbf{63.2} &      38.4 &          15.7 & \textbf{64.7} &      34.6 \\
  FATAvg   & &   &          58.3 &          86.1 &       7.4 &          42.6 &          84.6 &        2.2 &          38.4 &          84.1 &        2.2 &          24.6 &          47.4 &      127.9 &          15.4 &          57.8 &      38.4 &          10.7 &          57.9 &      34.6 \\
  FATProx  & & &          58.5 &          86.3 &       7.4 &          42.8 &          84.5 &        2.2 &          38.1 &          84.1 &        2.2 &          24.8 &          47.1 &      127.9 &          14.5 &          57.3 &      38.4 &          10.4 &          57.1 &      34.6 \\
  FATMeta  & & &          43.6 &          71.6 &       7.4 &          35.0 &          72.6 &        2.2 &          35.3 &          72.2 &        2.2 &           6.0 &          23.5 &      127.9 &           0.0 &          37.2 &      38.4 &           0.1 &          38.1 &      34.6 \\
  FedRob  & &  &          13.1 &          13.1 &       7.4 &          20.6 &          59.3 &       1032 &          17.7 &          48.9 &        645 &             - &             - &          - &             - &             - &          - &             - &             - &          - \\
  \bottomrule
  \end{tabular}
\end{table*}

\begin{minipage}{0.58\textwidth}
\begin{table}[H]
\vspace{-0.25in}
\centering
\tiny
\caption{Evaluation of RA with various attacks on Digits. $n$ and $\epsilon$ are the step number and the magnitude of attack.}
\label{tbl:eval_attacks}
\begin{tabular}{@{ }l|*{6}{c}|c@{ }}
  \toprule
    Attack &         PGD &       PGD &        MIA &       MIA &   AA &           LSA & SA \\
    $(n,~\epsilon)$ & (20,16) & (100,8) & (20,16) & (100,8) & (-, 8) & (7, -) & - \\
    \midrule
    FedRBN & \textbf{42.8} & \textbf{54.5} & \textbf{39.9} & \textbf{52.2} & \textbf{48.3}  &  \textbf{73.5} & 84.2 \\
    FATBN  &          28.6 &          41.6 &          27.0 &          39.7 &   31.0   &         64.0 & \textbf{84.6} \\
    FATAvg &          31.5 &          43.4 &          30.0 &          41.5 &    32.9  &          63.3  & 84.2 \\
  \bottomrule
\end{tabular}
\end{table}
\end{minipage}
\hfill
\begin{minipage}{0.4\textwidth}
\begin{table}[H]
\vskip -0.35in
  \caption{Compare FedRBN versus efficient FAT on Digits.
  }
  \label{tbl:bmk_single_source_prop_free}
  \tiny
  \centering
  \begin{tabular}{@{ }l|*{2}{@{ }c}*{2}{@{ }c}@{ }}
    \toprule
    & \multicolumn{2}{c}{20\% 3/5 AT domains} & \multicolumn{2}{c}{100\% Free AT} \\ %
    & \multicolumn{2}{c}{} &  \multicolumn{2}{c}{ \citep{shafahi2019adversarial} } \\
    & \multicolumn{1}{c}{FedRBN} &
    \multicolumn{1}{c}{FATAvg} &\multicolumn{1}{c}{FATAvg} & \multicolumn{1}{c}{FATBN} \\ %
  \midrule
   RA & \textbf{56.1} & 44.9 & 47.1 & 46.3  \\
   SA & \textbf{86.2} & 85.6 & 63.6 & 57.4 \\ 
   T  & 273 & \textbf{271} & 276 & 276 \\
   \bottomrule
   \end{tabular}
\vskip -0.2in
\end{table}
\end{minipage}

\vskip -0.2in
\section{Conclusion}
\vskip -0.1in
In this paper, we investigate a novel problem setting, federate propagating robustness, and propose a FedRBN algorithm that transfers robustness in FL through robust BN statistics.
Extensive experiments demonstrate the rationality and effectiveness of the proposed method, delivering both generalization and robustness in FL. 
We believe such a client-wise efficient robust learning can broaden the application scenarios of FL to users with diverse computation capabilities.

\bibliographystyle{abbrv}
\bibliography{main}

\clearpage

\appendix

\section{Additional Related Work}
\label{sec:related_appd}

This section reviews additional references in the areas of centralized adversarial learning and robustness transfer. 

\textbf{Efficient centralized adversarial training}.
A line of work has been motivated by similar concerns on the high time complexity of adversarial training.
For example, Zhang~\etal\ proposed to adversarially train only the first layer of a network which is shown to be more influential for robustness~\cite{zhang2019you}.
Free AT~\cite{shafahi2019adversarial} trades in some standard iterations (on different mini-batches) for estimating a cached adversarial attack while keeping the total number of iterations unchanged.
Wong~\etal\ proposed to randomly initialize attacks multiple times, which can improve simpler attacks more efficiently~\cite{wong2019fast}.
Most of existing efforts above focus on speed up the local training by approximated attacks that trade in either RA or SA for efficiency.
Instead, our method relocated the computation cost from budget-limited users to budget-sufficient users who can afford the expansive standard AT.
As result, the computation expense is indeed exempted for the budget-limited users and their standard performance is not significantly influenced.

\textbf{Robustness transferring}.
Our work is related to transferring robustness from multiple AT users to ST users.
For example, a new user can enjoy the transferrable robustness of a large model trained on ImageNet~\cite{hendrycks2019using}.
In order to improve the transferability, some researchers aim to align the gradients between domains by criticizing their distributional difference~\cite{chan2020what}.
A similar idea was adopted for aligning the logits of adversarial samples between different domains~\cite{song2019improving}.
By fine-tuning a few layers of a network, Shafahi~\etal\ shows that robustness can be transferred better than standard fine-tuning~\cite{shafahi2019adversarially}.
Rather than a central algorithm gathering all data or pre-trained models, our work considers a distributed setting when samples or their corresponding gradients can not be shared for distribution alignment.
Meanwhile, a large transferrable model is not preferred in the setting, because of the huge communication cost associating to transferring models between users.
Because of the non-iid nature of users, it is also hard to pick a proper user model, that works well on all source users, for fine-tuning on a target user.

\textbf{Locally adapted models for data heterogeneity}.
In the sense of modeling data heterogeneity, some prior work was done in adapting models for each local user \citep{smith2017federated,arivazhagan2019federated,fallah2020personalized,dinh2020personalized}.
For example, \cite{smith2017federated} studied the linear cases with regularization on the parameters, while we study a more general deep neural networks.
In addition, the work did not consider a data-dependent adversarial regularization for better robustness, but a regularization that is independent from the data.
Similarly, \cite{dinh2020personalized} regularizes the local parameters similar to the global model in $L_2$ distance, and \cite{fallah2020personalized} considers a meta-learning strategy instead.
A simpler method was proposed by \cite{arivazhagan2019federated} to only adapt the classifier head for different local tasks.
Since all the above methods do not adapt the robustness from global to local settings, we first study how the robustness can be propagated among users in this work.
\section{Additional Technical Details of FedRBN}

\subsection{FedRBN training with large $\lambda$}

As discussed in previous papers, BN is critical for stabilizing the convergence deep learning \cite{ioffe2015batch}.
When the source and target BNs are significantly distinguished from each other, then the transferred BN will result in large loss and therefore large gradient during optimization.
On the DomainNet dataset, we observe such great gradient explosion using transferred $\text{BN}_a$ when $\lambda$ is large (e.g., 0.5) in \cref{eq:PNC_loss}, and thus the FedRBN training fails to converge.
Though reducing $\lambda$ to a smaller value like 0.1 can smooth the convergence, it may lower the robustness gain.
Instead, we suggest using a gradient clipping technique to fix the issue and present an alternative user training algorithm in \cref{alg:user_train_clip} to replace \cref{alg:user_train}.
In \cref{alg:user_train_clip}, we highlight the difference from \cref{alg:user_train} and $\operatorname{CLAMP}(x, 10)$ scales the input $x$ by $10x/\norm{x}$ if $\norm{x}>10$.
In addition, we use a accumulative gradient to enable the separation of the gradients of the two losses in \cref{eq:PNC_loss}.

\begin{algorithm}[ht]  %
	\centering
	\small
	\caption{FedRBN: user training with clipping}
	\label{alg:user_train_clip}
	\begin{algorithmic}[1]
	\Require User budget type (AT or ST), initial parameters $\theta$ (AT) or $(\theta, \hat \mu_a, \hat \sigma^2_a)$ (ST) of the model $f$ from the server, adversary $A_{\epsilon}(\cdot)$, dataset $D$, learning rate $\eta$
    \For{mini-batch $\{(x, y)\}$ in $D$}
        \State $\ell_c \leftarrow \Ebb_{(x,y)} [ \ell_{CE}(f(x; \text{BN}_c),y)]$
        \State \textcolor{blue}{Initialize an accumulative gradient: $g\leftarrow 0$}
        \State Update $(\mu, \sigma^2)$ of $\text{BN}_c$
        \If{user budget type is AT}
            \State Perturb data $\tilde x \leftarrow A_{\epsilon}(f(x; \text{BN}_a))$
            \State $L \leftarrow {1\over 2} \left\{ \ell_c + \Ebb_{(\tilde x,y)} [ \ell_{CE}(f(\tilde x; \text{BN}_a),y)] \right\}$
            \State Update $(\mu_a, \sigma^2_a)$ of $\text{BN}_a$
        \Else
            \State Replace $\text{BN}_a$ parameters with $(\hat \mu_a, \hat \sigma^2_a)$
            \State \textcolor{blue}{$L_p \leftarrow \lambda \Ebb_{(x,y)} [ \ell_{CE}(f(x; \text{BN}_a),y)]$}
            \State \textcolor{blue}{$g \leftarrow g + \operatorname{CLAMP}(\frac{\partial L_p}{ \partial \theta}, 10)$}
            \State \textcolor{blue}{$L \leftarrow (1-\lambda) \ell_c$}
        \EndIf
        \State $g\leftarrow g + \frac{\partial L}{\partial \theta}$
        \State $\theta \leftarrow \theta - \eta g$
    \EndFor
    \State \textbf{Upload} $(\theta, \mu, \sigma^2, \mu_a, \sigma_a^2)$ (AT) or $(\theta, \mu, \sigma^2)$ (ST)
	\end{algorithmic}
\end{algorithm}

\subsection{Proof of Lemma \ref{thm:multi_src_domain}}
\label{thm:multi_src_domain:proof}

In this section, we use the notation $D$ for a dataset containing images and excluding labels.
To provide supervisions, we define a ground-truth labeling function $g$ that returns the true labels given images.
So as for distribution $\cD$.

First, in \cref{def:h_delta_h_div}, we define the $\cH\Delta \cH$-divergence that measures the discrepancy between two distributions.
Because the $\cH\Delta \cH$-divergence measures differences based on possible hypotheses (e.g., models), it can help relating model parameter differences and distribution shift.

\begin{definition}
    \label{def:h_delta_h_div}
    Given a hypothesis space $\cH$ for input space $\cX$, the $\cH$-divergence between two distributions $\cD$ and $\cD'$ is $d_{\cH} (\cD, \cD') \triangleq 2 \sup_{S\in \cS_{\cH}} \left| \operatorname{Pr}_{\cD}(S) - \operatorname{Pr}_{\cD'} (S) \right|$ where $\cS_{\cH}$ denotes the collection of subsets of $\cX$ that are the support of some hypothesis in $\cH$.
    The $\cH\Delta \cH$-divergence is defined on the symmetric difference space $\cH\Delta \cH \triangleq \{f(x) \oplus h'(x) | h,h' \in \cH \}$ where $\oplus$ denotes the XOR operation.
\end{definition}

Then, we introduce \cref{ass:bouned_div_adv} to bound the distribution differences caused by adversarial noise.
The reason for introducing such an assumption is that the adversarial noise magnitude is bounded and the resulting adversarial distribution should not differ from the original one too much.
Since all users are (or expected to be) noised by the same adversarial attacker $A_\epsilon (\cdot)$ during training, we can use $d_\epsilon$ to universally bound the adversarial distributional differences for all users.

\begin{assumption} \label{ass:bouned_div_adv}
Let $d_{\epsilon}$ be a non-negative constant governed by the adversarial magnitude $\epsilon$. For a distribution $\cD$, the divergence between $\cD$ and its corresponding adversarial distribution $\tilde \cD \triangleq \{ A_\epsilon(x) | x\sim \cD\}$ is bounded as $d_{\cH \Delta \cH} (\tilde \cD, \cD) \le d_{\epsilon}$.
\end{assumption}

Now, our goal is to analyze the \emph{generalization error} of model $\tilde f_t$ on the target adversarial distribution $\tilde \cD$, i.e., $L(\tilde f_t, \tilde \cD_t) = \Ebb_{\tilde x\sim \tilde \cD_t} [ | \tilde f_t(\tilde x) - g(\tilde x) | ]$.
Since we estimate $\tilde f_t$ by a weighted average, i.e., $\sum_i \alpha_i \tilde f_{s_i}$ where $\tilde f_{s_i}$ is the robust model on $D_{s_i}$, we can adapt the generalization error bound from \cite{peng2019federated} for adversarial distributions.
For consistency, we assume the AT users reside on the \emph{source} clean/noised domains while ST users reside on the \emph{target} clean/noised domains.
Without loss of generality, we only consider one target domain and assume one user per domain.

\begin{theorem}[Restated from Theorem 2 in \cite{peng2019federated}]
    \label{thm:fada}
    Let $\cH$ be a hypothesis space of $VC$-dimension $d$ and $\{\tilde D_{s_i}\}_{i=1}^N$, $\tilde D_t$ be datasets induced by samples of size $m$ drawn from $\{\tilde \cD_{s_i}\}_{i=1}^N$ and $\tilde \cD_t$, respectively.
    Define the estimated hypothesis as $\tilde f_t \triangleq \sum_{i=1}^N \alpha_i \tilde f_{s_i}$.
    Then, $\forall \alpha \in \RR^N_+$, $\sum_{i=1}^N \alpha_i = 1$, with probability at least $1 - p$ over the choice of samples, for each $f\in \cH$,
    \begin{align}
        L(f, \tilde \cD_t) &\le L(\tilde f_t, \tilde D_s) \notag \\
        &+ \sum\nolimits_{i=1}^N \alpha_i \left( {1\over 2} d_{\cH \Delta \cH} (\tilde \cD_{s_i}, \tilde \cD_t) + \tilde \xi_i \right) + C,
        \label{eq:fada}
    \end{align}
    where $C= 4 \sqrt{2d \log (2Nm) + \log(4/p) \over Nm}$, $\tilde \xi_i$ is the loss of the optimal hypothesis on the mixture of $\tilde D_{s_i}$ and $\tilde D_t$, and $\tilde D_s$ is the mixture of all source samples with size $Nm$.
    $d_{\cH \Delta \cH}(\tilde \cD_{s_i}, \tilde \cD_t)$ denotes the divergence between domain $s_i$ and $t$.
\end{theorem}

Based on \cref{thm:fada}, we may choose a weighting strategy by $\alpha_i \propto 1 / d_{\cH \Delta \cH} (\tilde \cD_{s_i}, \tilde \cD_t)$.
However, the divergence cannot be estimated due to the lack of the target adversarial distribution $\tilde \cD_t$.
Instead, we provide a bound by clean-distribution divergence in \cref{thm:multi_src_domain_formal}.

\begin{lemma}[Formal statement of \cref{thm:multi_src_domain}] \label{thm:multi_src_domain_formal}
    Suppose \cref{ass:bouned_div_adv} holds. 
    Let $\cH$ be a hypothesis space of $VC$-dimension $d$ and $\{D_{s_i}\}_{i=1}^N$, $D_t$ be datasets induced by samples of size $m$ drawn from $\{\cD_{s_i}\}_{i=1}^N$ and $\cD_t$.
    Let an estimated target (robust) model be $\tilde f_t = \sum_i \alpha_i \tilde f_{s_i}$ where $\tilde f_{s_i}$ is the robust model trained on $D_{s_i}$.
    Let $\tilde D_{s}$ be the mixture of source samples from $\{\tilde D_{s_i}\}_{i=1}^N$.
    Then, $\forall \alpha \in \RR_+^N$, $\sum_{i=1}^N \alpha_i =1$, with probability at least $1-p$ over the choice of samples, for each $f\in \cH$, the following inequality holds:
    \begin{align*}
        L(f, \tilde \cD_t) &\le L(\tilde f_t, \tilde D_{s}) + d_\epsilon \\
        &\quad+ \sum\nolimits_{i=1}^N \alpha_i \left({1\over 2} d_{\cH \Delta \cH} (\cD_{s_i}, \cD_t) + \xi_i \right) + C,
    \end{align*}
     where $C$ and $\tilde \xi_i$ are defined in \cref{thm:fada}. %
     $D_{s}$ is the mixture of all source samples with size $Nm$.
     $d_{\cH \Delta \cH} (\cD_{s_i}, \cD_t)$ is the divergence over clean distributions.
\end{lemma}

\begin{proof}
Notice that \cref{eq:fada} is a loose bound as $d_{\cH \Delta \cH} (\tilde \cD_{s_i}, \tilde \cD_t)$ is neither bounded nor predictable.
Differently, $d_{\cH \Delta \cH} (\cD_{s_i}, \cD_t)$ can be estimated by clean samples which is available for all users.
Thus, we can bound $d_{\cH \Delta \cH} (\tilde D_{s_i}, \tilde D_t)$ with $d_{\cH \Delta \cH} (D_{s_i}, D_t)$.
By \cref{ass:bouned_div_adv}, it is easy to attain
\begin{align}
    d_{\cH \Delta \cH} (\tilde D_{s_i}, \tilde D_t) 
    &\le d_{\cH \Delta \cH} (\tilde D_{s_i}, D_{s_i}) 
        + d_{\cH \Delta \cH} ( D_{s_i}, D_t) \notag \\
        &\quad + d_{\cH \Delta \cH} (D_{t}, \tilde D_t) \notag \\
    &\le 2 d_{\epsilon} + d_{\cH \Delta \cH} ( D_{s_i}, D_t), \label{eq:tri_d}
\end{align}
where we used the triangle inequality in the space measured by $d_{\cH\Delta \cH}(\cdot, \cdot)$.
Substitute \cref{eq:tri_d} into \cref{eq:fada}, and we finish the proof.
\end{proof}

In \cref{thm:multi_src_domain_formal}, we discussed the bound for a $f\in \cH$ (which also generalize to $\tilde f_t$) estimated by the linear combination of $\{\tilde f_{s_i}\}_{i}$.
In our algorithm, $\tilde f_t$ and $\tilde f_{s_i}$ both represent the models with noise BN layers, and they only differ by the BN layers.
Therefore, \cref{thm:multi_src_domain_formal} guides us to re-weight BN parameters according to the domain differences.
Specifically, we should upweight BN statistics from user $s_i$ if $d_{\cH \Delta \cH} (\cD_{s_i}, \cD_t)$ is large, vice versa.
Since $d_{\cH \Delta \cH} (\cD_{s_i}, \cD_t)$ is hard to estimate, we may use the divergence over empirical distributions, i.e., $d_{\cH \Delta \cH} (D_{s_i}, D_t)$ instead.

\subsection{Limitation and social impacts}
\label{sec:dbn_copy}

Federated learning has emerged as a very effective framework to involve more users in training and tends to benefit all users meanwhile.
However, the device heterogeneity is not well considered in the goal of FL, especially facing the risk from adversarial attackers that can revert the model predictions in slight image obfuscation.
Our work fills the gap by developing a novel algorithm that shares robustness from resource-limited devices to those that are powerful enough to do adversarial training.
We believe that our work can ubiquitously benefit many low-energy devices and encourage fairness in machine learning.

Though our method can effectively and efficiently propagate robustness, a more complicated real-world environment could be considered.
For instance, the hardware capability may not be aware of the server resulting in the increased hardness of directional propagation.
We believe our work could be the starting point for resolving these complicated problems and we will be devoted to working them in the future.
\section{Additional Empirical Study Results}
\label{sec:app:exp}

We provide more details about our experiments in \cref{sec:app:exp_details} and additional evaluation results in the rest sections.
To ease the reading, we summarize the content as follows with the referred section numbers in brackets.
\preprintrm{(1)~\textbf{Qualitative studies}}Qualitative studies show that our method converges faster than the baselines (\ref{sec:conv_hyper}); $\lambda$ can trade off the robustness for standard accuracy like the AT coefficient; our BN-centered principle is well motivated by the significance of the concerned gaps in BN statistics and representations (\ref{sec:app:dist_bias_and_bn}, \ref{sec:app:rep}).\preprintrm{; the advantage of our method is fairly present on all domains rather than a biased improvement (\ref{sec:app:domain}); and fewer training samples or lower validation ratios will not hurt the advantage of our method (\ref{sec:data_size_impact}).
(2)~\textbf{Generalization to different settings}. We also conduct very detailed experiments to show the generalization of the superior performance of our method to more datasets (\ref{sec:office_results}), other models (\ref{sec:app:resnet}) and varying FL configurations in terms of contact ratios (the number of executed users per communication round), the scalability with more users, e.g., 200 users (\ref{sec:app:fl_conf}), variable data sizes per user (\ref{sec:app:unequal_size}).}
We also extend the experiments of \cref{fig:frp_varying} to DomainNet to show the generalization of the conclusions (\ref{sec:app:ext_fig}).

\subsection{Experiment details}
\label{sec:app:exp_details}

\textbf{Data}. By default, we use $30\%$ data of Digits for training.
Datasets for all domains are truncated to the same size following the minimal one.
In addition, we leave out $50\%$ of the training set for validation for Digits and $60\%$ for DomainNet.
Test sets are preset according to the benchmarks in~\cite{li2020fedbn}.
Models are selected according to the validation accuracy.
To be efficient, we validate robust users with RA while non-robust users with SA.
We use a large ratio of the training set for validation, because the very limited sample size for each user will result in biased validation accuracy.
When selecting a subset of domains for AT users, we select the first $n$ domains by the order: (MN, SV, US, SY, MM) for \textsc{Digits}, and (R, C, I, P, Q, S) for \textsc{DomainNet}.
Some samples are plotted in \cref{fig:dataset_samples} to show the visual difference between domains.

\begin{figure*}[h]
    \centering
    \includegraphics[width=\textwidth]{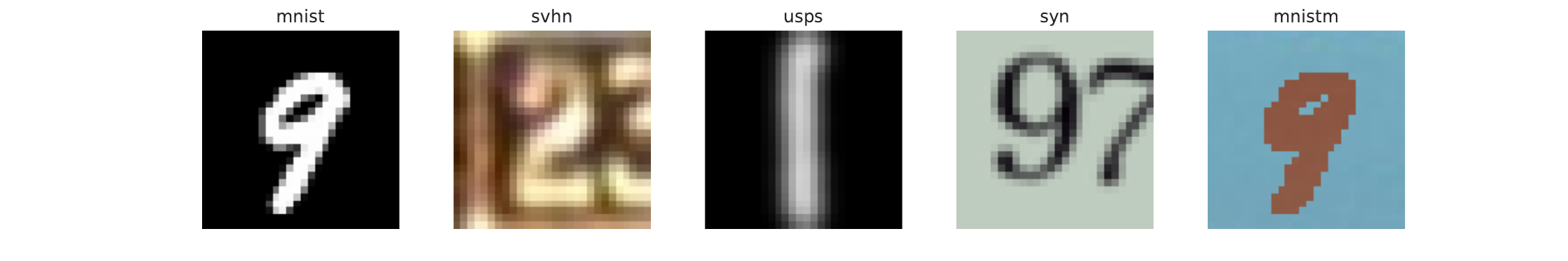}
    \includegraphics[width=\textwidth]{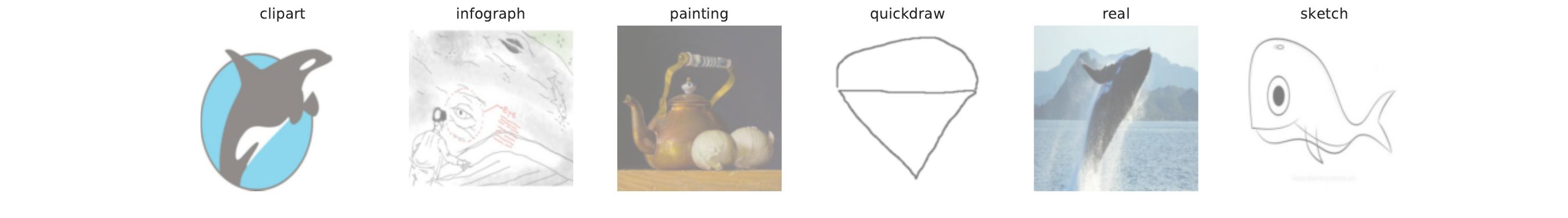}
    \caption{Visualization of samples.}
    \label{fig:dataset_samples}
\end{figure*}

\textbf{Hyper-parameters}.
The only hyper-parameter here is the $\lambda$ in PNC loss.
As PNC loss mimick the behavior of an adversarial loss $(\ell_{\text{CE}} + \ell_a)/2$, we follow the common practice to use $\lambda=0.5$ which could provide a generally fair balance between accuracy and robustness.

\begin{table}[ht]
    \centering
    \caption{Network architecture for Digits dataset.}
    \label{tbl:digits}
    \scriptsize
    \begin{tabular}{c|c}
    \toprule
        Layer & Details \\
    \midrule
        \multicolumn{2}{c}{\textbf{feature extractor}} \\
        \midrule
        conv1 & Conv2D(64, kernel size=5, stride=1, padding=2) \\
        bn1   & DBN2D, RELU, MaxPool2D(kernel size=2, stride=2) \\
        conv2 & Conv2D(64, kernel size=5, stride=1, padding=2) \\
        bn2   & DBN2D, ReLU, MaxPool2D(kernel size=2, stride=2) \\
        conv3 & Conv2D(128, kernel size=5, stride=1, padding=2) \\
        bn3   & DBN2D, ReLU \\
        \midrule
        \multicolumn{2}{c}{\textbf{classifier}} \\
        \midrule
        fc1   & FC(2048) \\
        bn4   & DBN2D, ReLU \\
        fc2   & FC(512) \\
        bn5   & DBN1D, ReLU \\
        fc3   & FC(10) \\
    \bottomrule
    \end{tabular}
\end{table}

\begin{table}[ht]
    \caption{Network architecture for DomainNet dataset.}
    \label{tbl:AlexNet}
    \centering
    \scriptsize
    \begin{tabular}{c|c}
    \toprule
        Layer & Details \\
    \midrule
        \multicolumn{2}{c}{\textbf{feature extractor}} \\
        \midrule
        conv1 & Conv2D(64, kernel size=11, stride=4, padding=2) \\
        bn1 & DBN2D, ReLU, MaxPool2d(kernel size=3, stride=2) \\
        conv2 & Conv2D(192, kernel size=5, stride=1, padding=2) \\
        bn2 & DBN2D, ReLU, MaxPool2d(kernel size=3, stride=2)  \\
        conv3 & Conv2D(384, kernel size=3, stride=1, padding=1)  \\
        bn3 & DBN2D, ReLU  \\
        conv4 & Conv2D(256, kernel size=3, stride=1, padding=1)  \\
        bn4 & DBN2D, ReLU  \\
        conv5 & Conv2D(256, kernel size=3, stride=1, padding=1)  \\
        bn5 & DBN2D, ReLU, MaxPool2d(kernel size=3, stride=2)  \\
        avgpool & AdaptiveAvgPool2d(6, 6)  \\
      \midrule
      \multicolumn{2}{c}{\textbf{classifier}} \\
      \midrule
        fc1 & FC(4096) \\
        bn6 & DBN1D, ReLU  \\
        fc2 & FC(4096)  \\
        bn7 & DBN1D, ReLU \\
        fc3 & FC(10)  \\
    \bottomrule
    \end{tabular}
\end{table}

\textbf{Network architectures} for \textsc{Digits} and \textsc{DomainNet} are listed in \cref{tbl:AlexNet,tbl:digits}.
For the convolutional layer (Conv2D or Conv1D), the first argument is the number channel. For a fully connected layer (FC), we list the number of hidden units as the first argument.

\textbf{Training}. 
Following \cite{li2020fedbn}, we conduct federated learning with $1$ local epoch and batch size $32$, which means users will train multiple iterations and communicate less frequently.
Without specification, we let all users participant in the federated training at each round.
Input images are resized to $256 \times 256$ for \textsc{DomainNet} and $28\times 28$ for \textsc{Digits}.
SGD (Stochastic Gradient Descent) is utilized to optimize models locally with a constant learning rate $10^{-2}$.
Models are trained for $300$ epochs by default.
For FedMeta, we use the $0.001$ learning rate for the meta-gradient descent and $0.02$ for normal gradient descent following the published codes from \cite{dinh2020personalized}.
We fine-tune the parameters for \textsc{DomainNet} such that the model can converge fast.
FedMeta converges slower than other methods, as it uses half of the batches to do the one-step meta-adaptation.
We do not let FedMeta fully converge since we have to limit the total FLOPs for a fair comparison.
FedRob fails to converge because locally estimated affine mapping is less stable with the large distribution discrepancy.

We implement our algorithm and baselines by PyTorch.
The FLOPs are computed by \texttt{thop} package in which the FLOPs of common network layers are predefined \footnote{Retrieve the \texttt{thop} python package from \url{https://github.com/Lyken17/pytorch-OpCounter}.}.
Then we compute the times of forwarding (inference) and backward (gradient computing) in training.
Accordingly, we compute the total FLOPs of the algorithm.
Because most other computation costs are relatively minor compared to the network forward/backward, these costs are ignored in our reported results.

\subsection{Convergence}
\label{sec:conv_hyper}

\begin{wrapfigure}{r}{0.4\textwidth}
    \centering
    \includegraphics[width=0.32\textwidth]{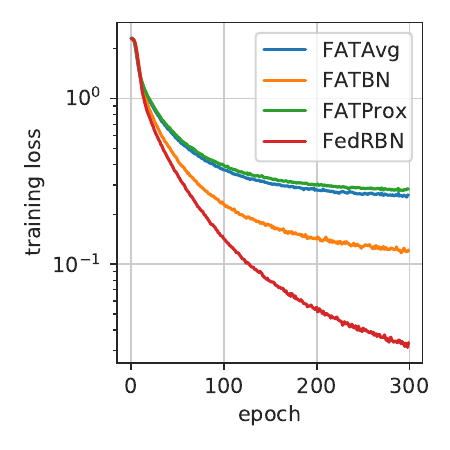}
    \vspace{-0.15in}
    \caption{The convergence curves.}
    \label{fig:param_sens}
    \vspace{-0.2in}
\end{wrapfigure}
The plot in \cref{fig:param_sens} shows convergence curves of different competing algorithms.
Since FedRBN is similar as FATBN in handling client heterogeneity, FATBN and FedRBN have similar convergence rates that are faster than others.
We see that FedRBN converges even faster than FATBN. A possible reason is that DBN decouples the normal and adversarial samples,
the representations after BN layers will be more consistently distributed among non-iid users.

\subsection{Effect of the PNC parameter $\lambda$}
\label{sec:app:trade_off_pnc}

In \cref{fig:trade_off}, we provide a detailed study on the parameter $\lambda$ based on the Digits configuration in \cref{tbl:bmk_single_source_prop}.
First, the effect of $\lambda$ is consistent on RA for all partial AT settings, where a larger $\lambda$ leads to better RA. 
Second, few PNC loss can enhance SA.
To understand this, we present \cref{fig:pen_fea_viz_BNa_FRP-full-OtoM_ex} where the unadapted $\text{BN}_a$ provides less discrimination on the clean examples compared to the adapted one (\cref{fig:pen_fea_viz_FRP_BNc-BNa_ex}), due to the domain bias.
Thus, by fixing such a bias, the proposed PNC loss with non-zero $\lambda$ improves the accuracy on clean examples. 
Last, because of the conflicting nature of RA and SA~\cite{tsipras2019robustness}, upweighting the surrogate adversarial loss (PNC) leads to better RA but worse SA.

\begin{figure}[ht]
    \centering
    \begin{subfigure}{0.48\textwidth}
        \centering
        \includegraphics[width=0.8\textwidth]{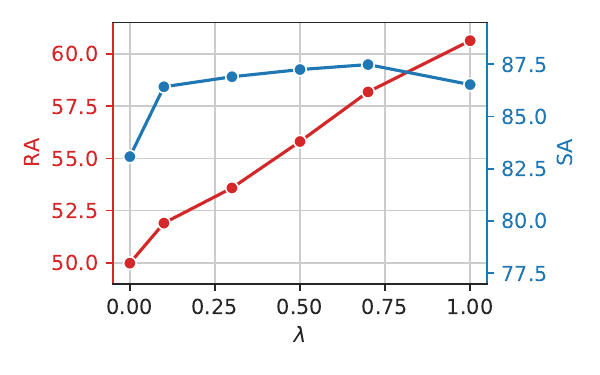}
        \vspace{-0.15in}
        \caption{20\% AT users}
    \end{subfigure}
    \hfil
    \begin{subfigure}{0.48\textwidth}
        \centering
        \includegraphics[width=0.8\textwidth]{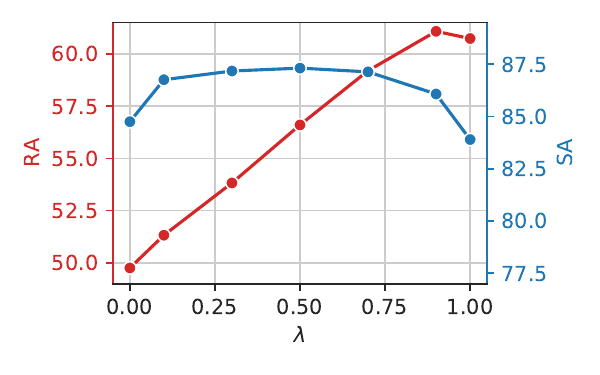}
        \vspace{-0.15in}
        \caption{MNIST-domain AT users}
    \end{subfigure}
    \caption{Evaluate the effect of PNC coefficient $\lambda$ on the Digits dataset. The single domain in (b) is MNIST.}
    \label{fig:trade_off}
\end{figure}

\subsection{BN statistic heterogeneity}
\label{sec:app:dist_bias_and_bn}

\begin{wrapfigure}{r}{0.5\textwidth}
    \centering
    \includegraphics[width=0.22\textwidth]{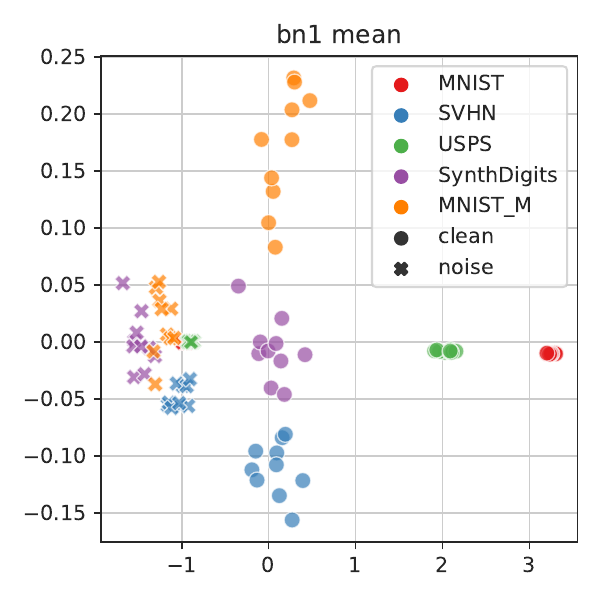} \hfil
    \includegraphics[width=0.22\textwidth]{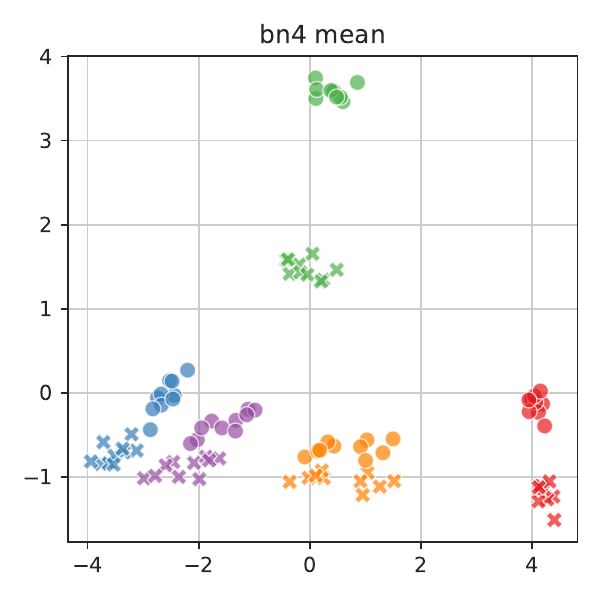} %
    \vskip -0.1in
    \caption{Visualization of users' BN statistics by PCA, colored by users' domains and marked by $\text{BN}_a$ (cross) and $\text{BN}_c$ (circle). The BNs parameters (mean and variance) after the first linear layer are extracted from a Digits model trained by LBN+DBN on a 100\% AT FRP setting. Figures for more datasets and layers in \cref{fig:bn_pca_ex}.}
    \label{fig:bn_pca}
    \vskip -0.4in
\end{wrapfigure}

We present the BN statistics in \cref{fig:bn_pca} and more figures in \cref{fig:bn_pca_ex}.
In \cref{fig:bn_pca}, both the domain biases and the adversarial ones are not neglectable, motivating us to combine the two techniques in FRP.
In addition, we notice that the latter biases are more significant in the shallow layer (\texttt{bn1}) where adversarial BN statistics from all domains are gathered together away from their clean statistics.

\subsection{More figures of representations}
\label{sec:app:rep}

We provide more representation visualization in \cref{fig:pen_fea_viz_ex}.

\subsection{More federated configurations}
\label{sec:app:fl_conf}

\begin{wraptable}{R}{0.4\textwidth}
    \vskip -0.2in
    \centering
    \caption{Evaluation with different FL configurations}
    \label{tab:eval_fl_conf}
    \centering
    \scriptsize
    \begin{tabular}{lll|cc}
    \toprule
    $B$ & $E$ & method &         RA &         SA \\
    \midrule
    10  & 1 & FATBN & 50.9 & \textbf{83.9} \\
        & 1 & FedRBN & \textbf{60.0} & 82.8 \\
    \midrule
    10  & 4 & FATBN & 42.0 & 75.8 \\
        & 4 & FedRBN & \textbf{56.3} & \textbf{76.1} \\
    \midrule
    10  & 8 & FATBN & 30.9 & 63.1 \\
        & 8 & FedRBN & \textbf{53.4} & \textbf{68.4} \\
    \midrule
    50  & 1 & FATBN & 37.0 & \textbf{85.8} \\
        & 1 & FedRBN & \textbf{53.2} & 84.5 \\
    \midrule
    100 & 1 & FATBN & 35.7 & \textbf{85.3} \\
        & 1 & FedRBN & \textbf{53.0} & 83.8 \\
    \bottomrule
    \end{tabular}
\end{wraptable}

We also evaluate our method against FedBN with different federated configurations of local epochs $E$ and batch size $B$.
We constrain the parameters by $E\in \{1, 4, 8\}$ and $B\in \{10, 50, 100\}$.
The $20\%$ $3/5$ domain FRP setting is adopted with \textsc{Digits} dataset.
In \cref{tab:eval_fl_conf},  %
the competition results are consistent that our method significantly promotes robustness over FedBN.
We also observe that both our method and FedBN prefer a smaller batch size and fewer local epochs for better RA and SA.
In addition, our method drops less RA when $E$ is large or batch size increases.

\textbf{Partial participants}.
In reality, we cannot expect that all users are available for training in each round.
Therefore, it is important to evaluate the federated performance when only a few users can contribute to the learning.
To simulate the scenario, we uniformly sample a number of users without replacement per communication round.
Only these users will train and upload models.
In \cref{fig:train_pr_nuser}, RA and SA are reported against the number of selected users.
We observe that SA is barely affected by the partial involvement, while RA increases by fewer users per round.
Since the actual update steps in the view of the global server are reduced with lower contact ratios, the result is consistent with \cref{tab:eval_fl_conf}, where smaller batch sizes or fewer local steps lead to better robustness.

\begin{figure*}[ht]
    \vspace{-0.1in}
    \centering
    \includegraphics[width=0.27\textwidth]{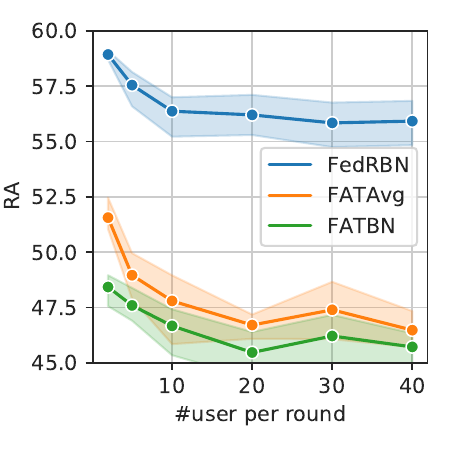} \hfil
    \includegraphics[width=0.27\textwidth]{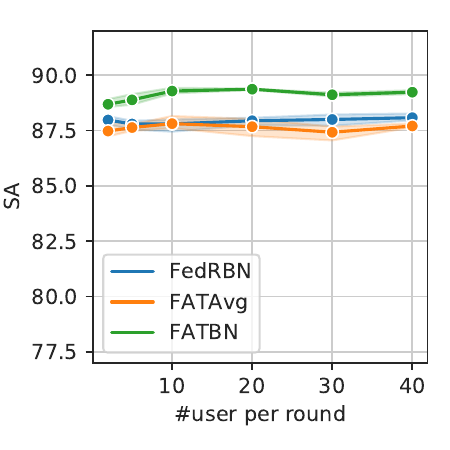} \hfil
    \includegraphics[width=0.34\textwidth]{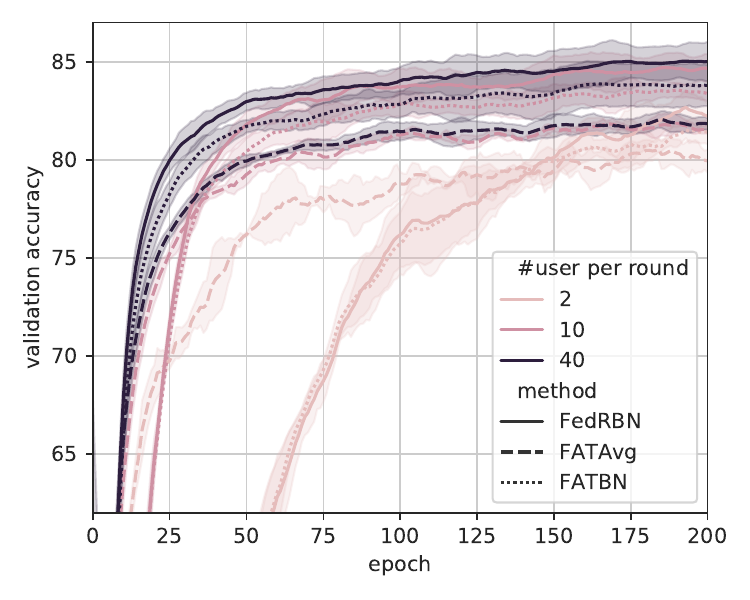}
    \caption{Vary the number of involved users per communication round. The validation accuracy is computed by averaging users' accuracy. For AT users, the RA is used while SA is used for ST users.}
    \label{fig:train_pr_nuser}
\end{figure*}

\textbf{Scalability with more users}.
Since our method has the similar training/communication strategy as FATBN or FATAvg (except switching and copying BN which are quite lightweight), the federation of FedRBN and its complexity scale up to more users like FATBN or FATAvg who are widely used scalable implementations.
To empirically evaluate the scalability of our method versus FATBN and FATAvg, we experiment with more clients given the Digits dataset.
With the same total training samples, we re-distribute the data to different numbers of clients in a non-uniform manner.
In \cref{fig:train_pd_nuser_trade_off}, we evaluate the RA and SA by increasing the total number of users, including 25, 50, 150, 200.
In each communication round, 50\% randomly selected users will upload their trained models.
The trend shows that both RA and SA will be lower when samples are distributed to more clients.
Despite the degradation, our method maintains advantages in RA consistently.

\begin{figure}[ht]
    \vspace{-0.15in}
    \centering
    \includegraphics[width=3.6cm]{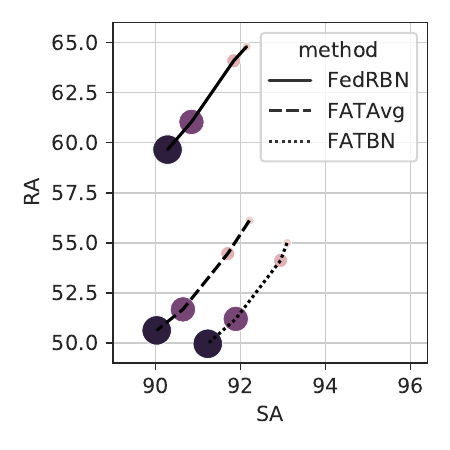}
    \hfil
    \includegraphics[width=4.8cm]{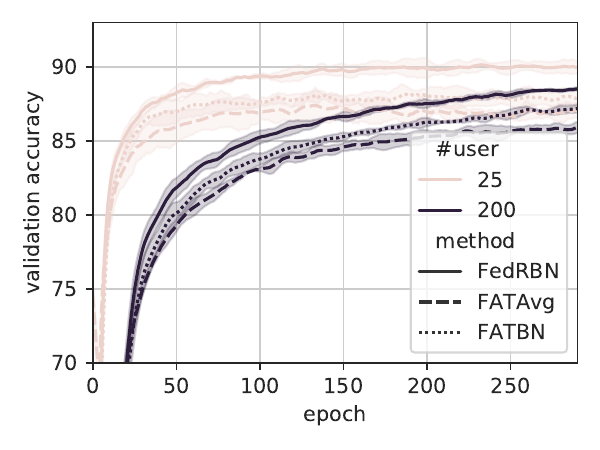}
    \caption{Robustness and accuracy by the increasing total number of users as 25, 50, 150, and 200. The larger scatter in the left figure indicates more users.}
    \label{fig:train_pd_nuser_trade_off}
\end{figure}
    
In \cref{fig:train_pd_nuser_trade_off}, we also demonstrate that our method converges faster than baselines either with fewer or more users.
The validation accuracy is computed by averaging users' accuracy when RA is used for AT users and SA for ST users.
As observed, when data are more concentrated in a few users (i.e. smaller numbers of users), the convergence will be faster.
The result is natural for most non-iid federated learning problems.
For example, \cite{li2020convergence} proved that more clients will result in worse final losses and slower convergence.

\begin{table}[ht]
    \caption{Comparison to robustness transferring by fine-tuning (FT).}
    \label{tab:ft}
    \centering
    \scriptsize
    \begin{tabular}{l|*{4}{c}}
    \toprule
           &  \# FT iterations & \# freeze layers &   RA &   SA \\
    \midrule
    FedRBN &       - & 0 & \textbf{53.1} & 84.4 \\
    FedAvg &       - & 0 & 44.7 & \textbf{85.7} \\
    \midrule
    FedAvg+FT &       200 & 0 & 39.2  & 83.6 \\
    FedAvg+FT &       200 & 3 & 31.6 & 78.2 \\
    FedAvg+FT &       200 & 4 & 29.8 & 74.7 \\
    FedAvg+FT &       200 & 5 & 31.5 & 66.1 \\
    \midrule
    FedAvg+FT &       100 & 0 & 40.6 & 83.4 \\
    FedAvg+FT &       100 & 3 & 32.0 & 77.5 \\
    FedAvg+FT &       100 & 4 & 31.5 & 72.9 \\
    FedAvg+FT &       100 & 5 & 31.5 & 64.5 \\
    \midrule
    FedAvg+FT &        20 & 0 & 40.6 & 79.6 \\
    FedAvg+FT &        20 & 3 & 33.4 & 73.8 \\
    FedAvg+FT &        20 & 4 & 31.9 & 66.8 \\
    FedAvg+FT &        20 & 5 & 31.9 & 62.2 \\
    \bottomrule
    \end{tabular}
\end{table}

\preprintrm{

\subsection{Comparison to robustness transferring by fine-tuning}

As an alternative to FRP, fine-tuning (FT) the federated-trained models on target users can enjoy even better efficiency than FedRBN.
Here, we first train AT users by FedAvg for 300 epochs.
Note that we do not adopt FedBN because FedBN will not output a single model for adapting to new users.
Then, the model is used for initializing the models for ST users.
These ST users will be trained by FedAvg for a given number of FT iterations.
Still, we adopt the $20\%$ $3/5$ domain FRP setting on the \textsc{Digits} dataset.
In \cref{tab:ft}, we see that such a fine-tuning does not improve the robustness (RA).

}

\subsection{Experiments in \cref{fig:highlight}}
\label{sec:app:exp_highlight}

Though the results in \cref{fig:highlight} have been reported in previous experiments.\preprintrm{, we re-summarize the results in \cref{tbl:highlight} for ease of reading.}
The basic setting follows the previous experiments on the Digits dataset.
We construct different portions of AT users by \emph{in-domain} or \emph{out-domain} propagation settings.
When robustness is propagated in domains, we sample AT users in each domain by the same portion and leave the rest as ST users.
When robustness is propagated out of domains, all users from the last two domains will not be adversarially trained and gain robustness from other domains.
Concretely, we add the FedRBN without copy propagation (\texttt{FedRBN w/o prop}) in the table, to show the propagation effect.
\texttt{FedRBN w/o prop} outperforms the baselines only when the AT-user portion is more than 60\%.
Meanwhile, due to the lack of copy propagation, the RA is much worse than the propagated \texttt{FedRBN}.
Unless no AT user presents in the federated learning, FedRBN always outperforms baselines.

\subsection{Extending experiments of \cref{fig:frp_varying}}
\label{sec:app:ext_fig}

In \cref{fig:vary_frp}, we evaluate methods in varying FRP settings and FedRBN beats the strongest baselines consistently

\begin{figure}[ht]
    \centering
    \begin{subfigure}{0.37\textwidth}
        \centering
        \includegraphics[width=0.8\textwidth]{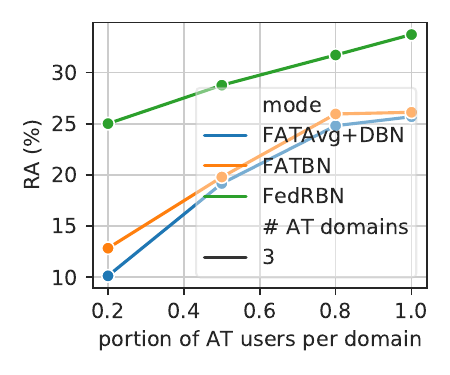}
        \vspace{-0.15in}
        \caption{Vary ratio of AT users}
    \end{subfigure}
    \hfil
    \begin{subfigure}{0.37\textwidth}
        \centering
        \includegraphics[width=0.8\textwidth]{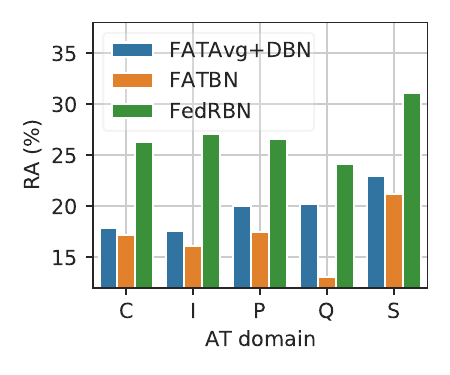}
        \vspace{-0.15in}
        \caption{Vary domain of AT users}
    \end{subfigure}
    \caption{Evaluating FRP performance with varying FRP settings on DomainNet. The x-axis of (b) represents the first letter of each domain.}
    \label{fig:vary_frp}
\end{figure}

\begin{figure}[ht]
    \centering
    \includegraphics[width=0.45\textwidth]{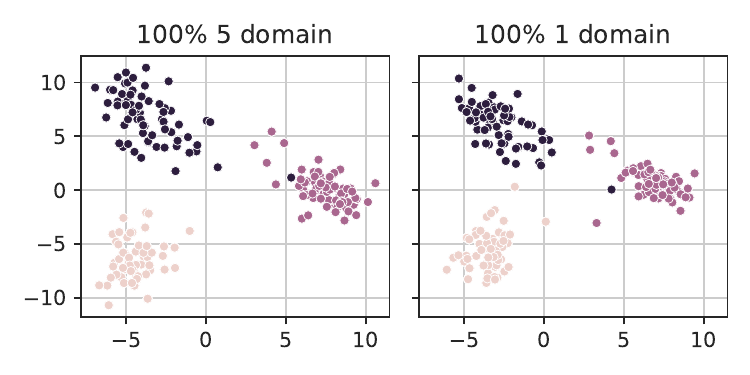}
    \includegraphics[width=0.45\textwidth]{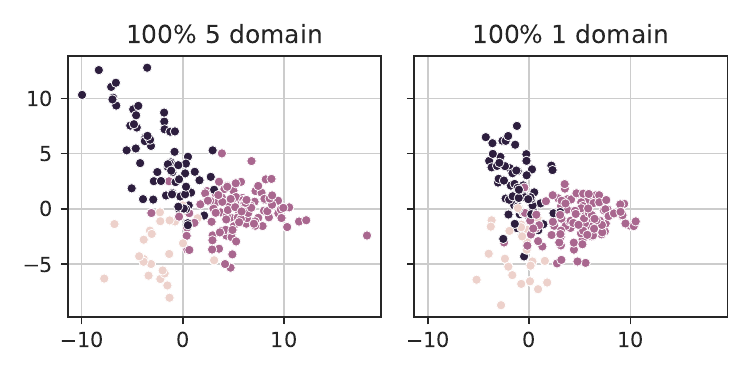}
    \includegraphics[width=0.45\textwidth]{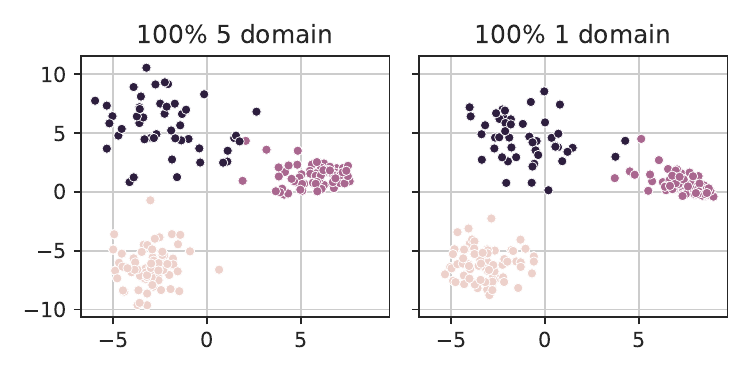}
    \includegraphics[width=0.45\textwidth]{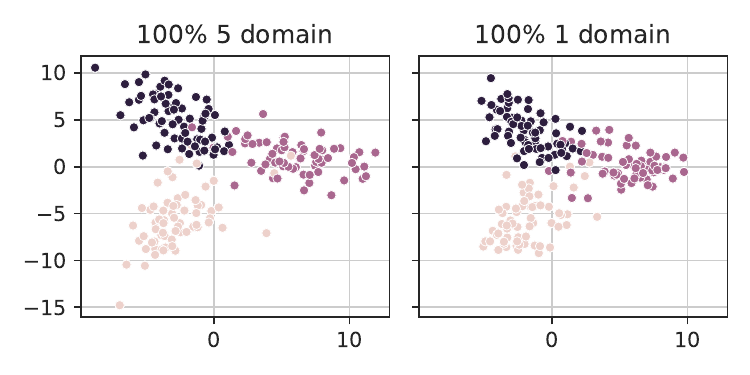}
    \includegraphics[width=0.45\textwidth]{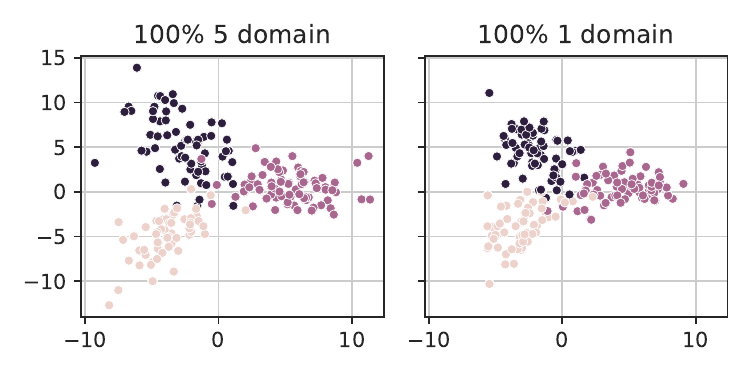}
    \caption{Penultimate layer representation visualized by projecting 300 randomly-selected samples into the first three classes in standard Digits model following \cite{muller2019when}. From the top to the bottom, these domains are visualized: MNIST, SVHN, SynthDigits, and MNIST-M.}
    \label{fig:pen_fea_viz_ex_0}
\end{figure}

\begin{figure*}[t]
    \centering
    \begin{subfigure}{0.49\textwidth}
        \centering
        \includegraphics[width=0.95\textwidth]{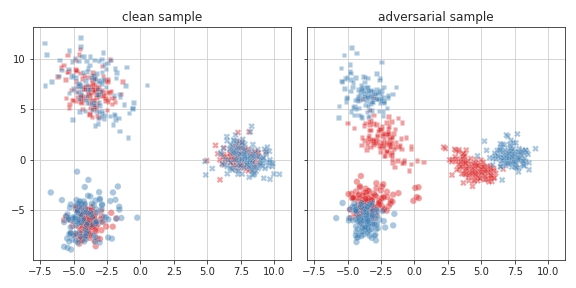}
        \includegraphics[width=0.95\textwidth]{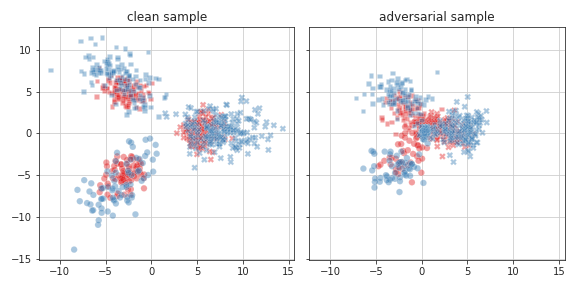}
        \includegraphics[width=0.95\textwidth]{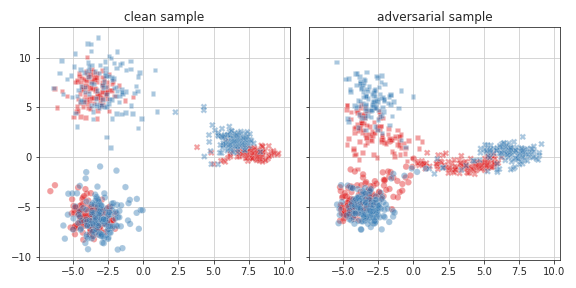}
        \includegraphics[width=0.95\textwidth]{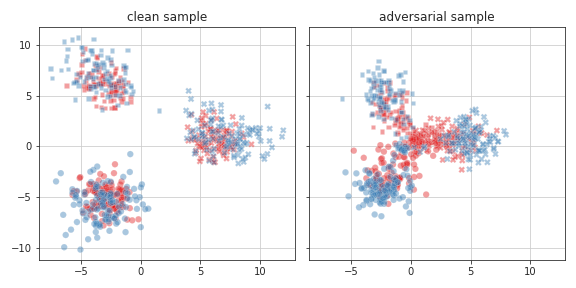}
        \includegraphics[width=0.95\textwidth]{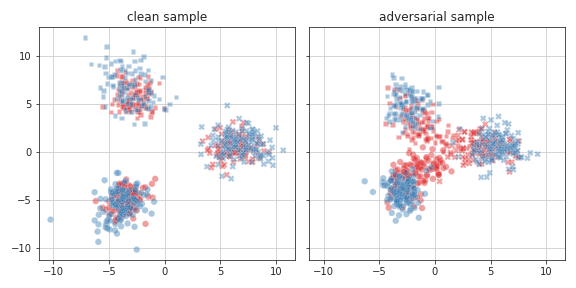}
        \caption{Representations are computed with \textcolor{red}{$BN_c$} or \textcolor{blue}{$BN_a$}.}
        \label{fig:pen_fea_viz_FRP_BNc-BNa_ex}
    \end{subfigure}
    \begin{subfigure}{0.49\textwidth}
        \centering
        \includegraphics[width=0.95\textwidth]{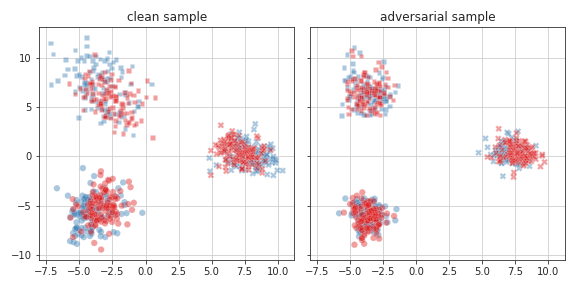}
        \includegraphics[width=0.95\textwidth]{fig/FedRBNn/pen_fea_viz_Digits_SVHN_BNa_FRP-full-OtoM.png}
        \includegraphics[width=0.95\textwidth]{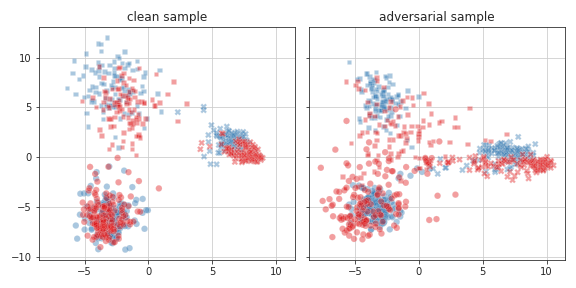}
        \includegraphics[width=0.95\textwidth]{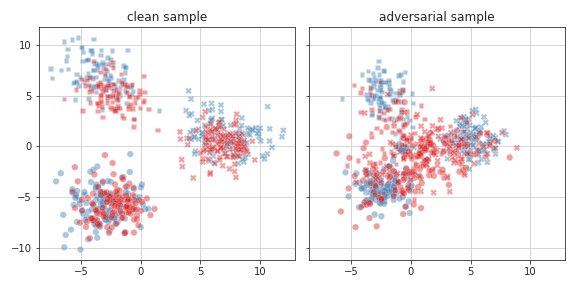}
        \includegraphics[width=0.95\textwidth]{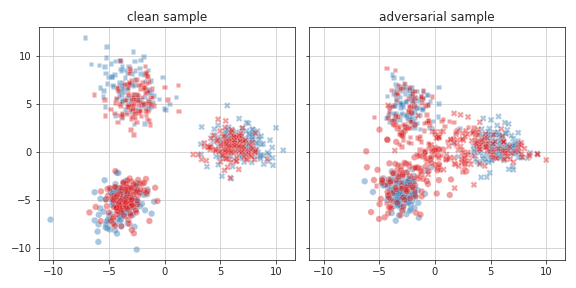}
        \caption{Representations are computed by \textcolor{blue}{trained} or \textcolor{red}{transferred} $BN_a$. }
        \label{fig:pen_fea_viz_BNa_FRP-full-OtoM_ex}
    \end{subfigure}
    \caption{Penultimate layer representations visualized by a Digits model and SVHN-domain users. From the top to the bottom, figures are plotted by domains: MNIST, SVHN, USPS, SynthDigits and MNIST-M in the Digits dataset.}
    \label{fig:pen_fea_viz_ex}
\end{figure*}

\begin{figure*}
    \centering
    \begin{subfigure}{0.49\textwidth}
        \centering
        \includegraphics[width=0.36\textwidth]{fig/FedRBNn/digits_PCA_bn1_bn_mean} \hfil
        \includegraphics[width=0.36\textwidth]{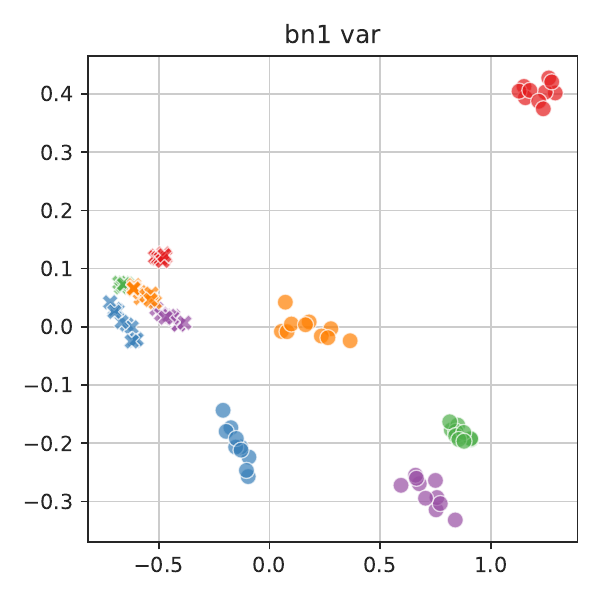} \hfil
    
        \includegraphics[width=0.36\textwidth]{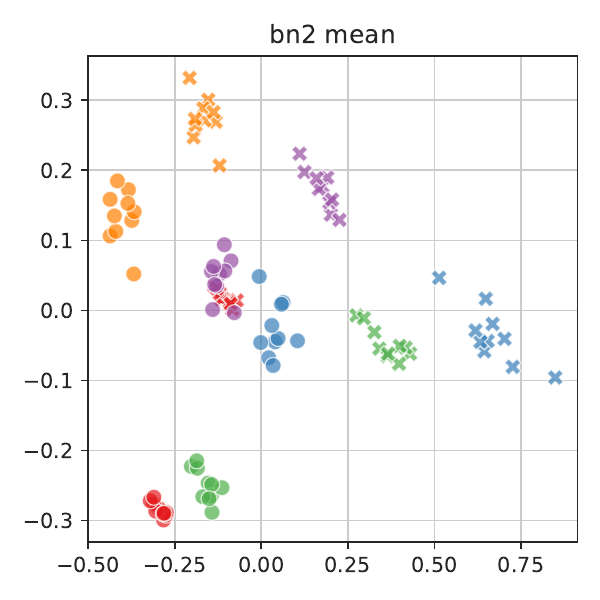} \hfil
        \includegraphics[width=0.36\textwidth]{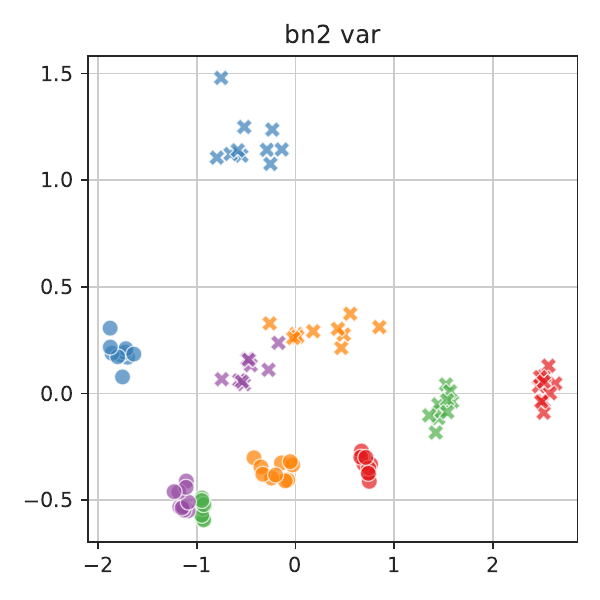} \hfil
    
        \includegraphics[width=0.36\textwidth]{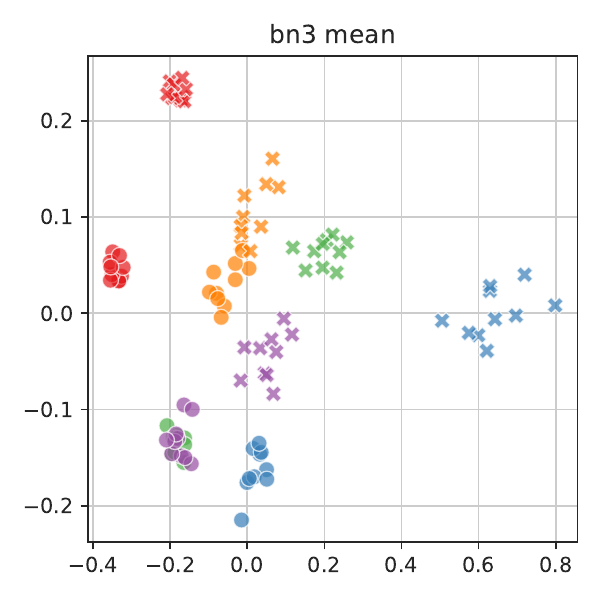} \hfil
        \includegraphics[width=0.36\textwidth]{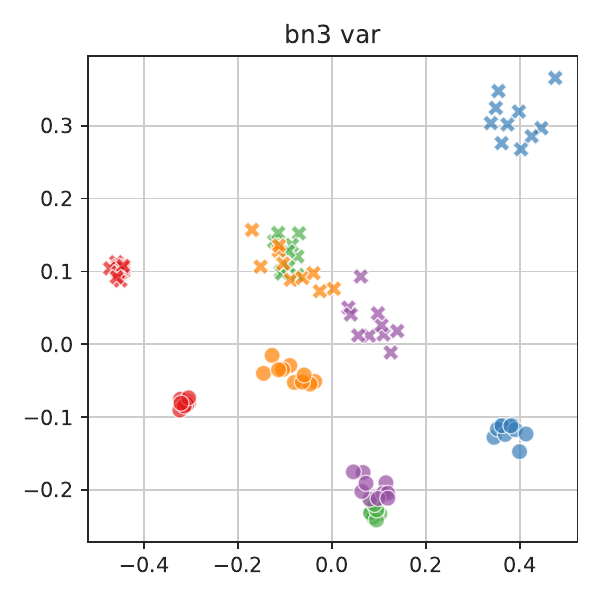} \hfil

        \includegraphics[width=0.36\textwidth]{fig/FedRBNn/digits_PCA_bn4_bn_mean} \hfil
        \includegraphics[width=0.36\textwidth]{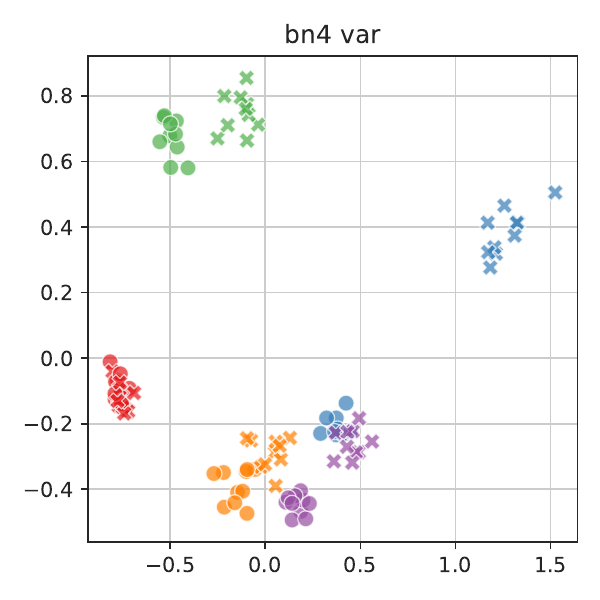} \hfil

        \includegraphics[width=0.36\textwidth]{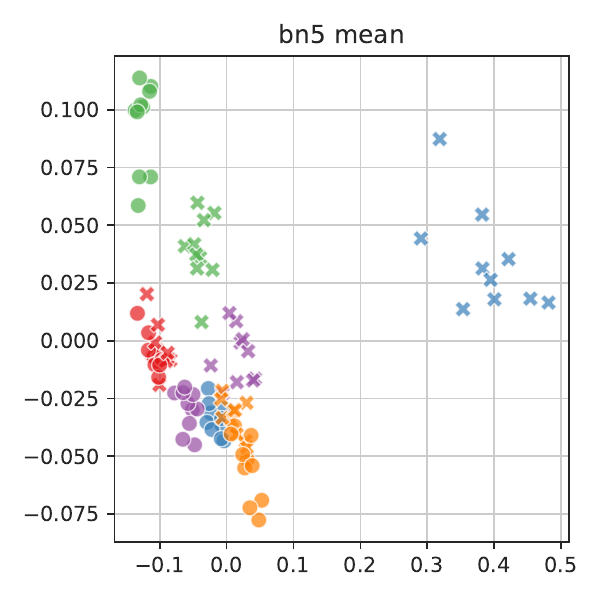} \hfil
        \includegraphics[width=0.36\textwidth]{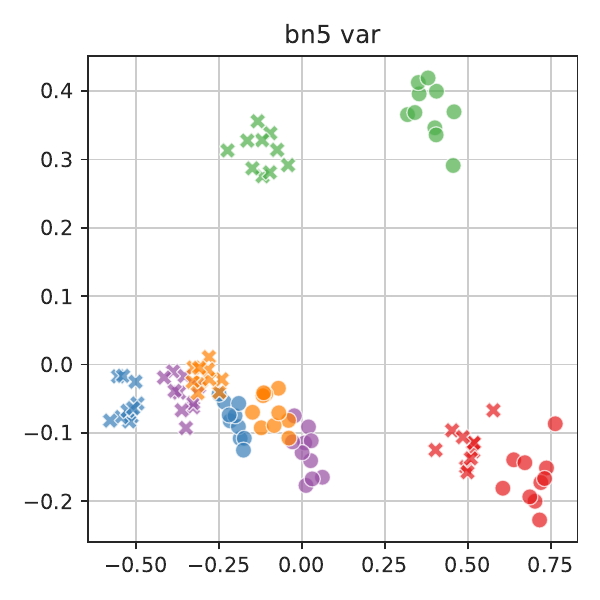} \hfil
        \caption{Digits dataset}
    \end{subfigure}
    \begin{subfigure}{0.49\textwidth}
        \centering
        \includegraphics[width=0.36\textwidth]{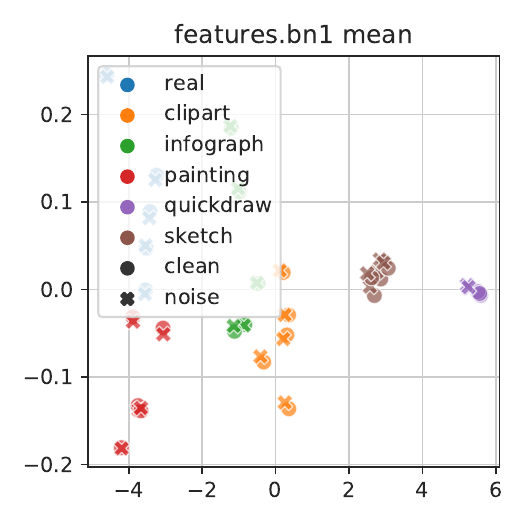} \hfil
        \includegraphics[width=0.36\textwidth]{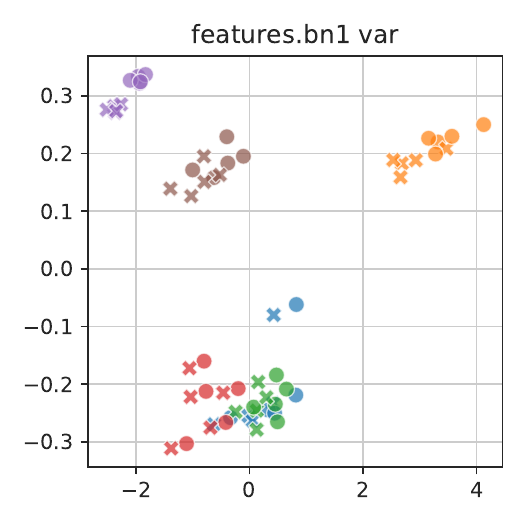}
    
        \includegraphics[width=0.36\textwidth]{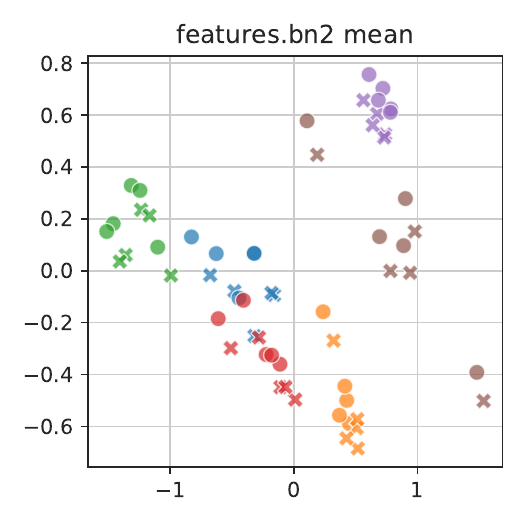} \hfil
        \includegraphics[width=0.36\textwidth]{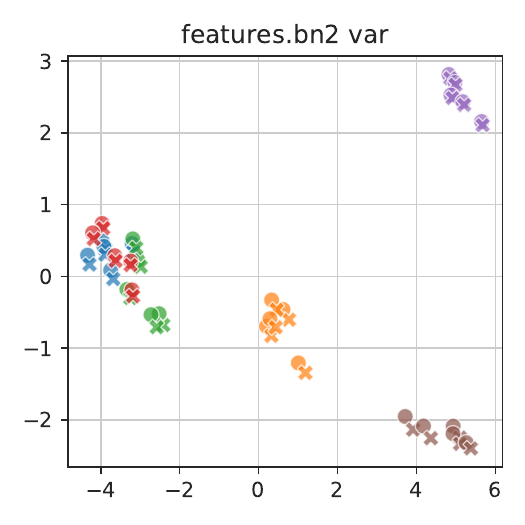}
    
        \includegraphics[width=0.36\textwidth]{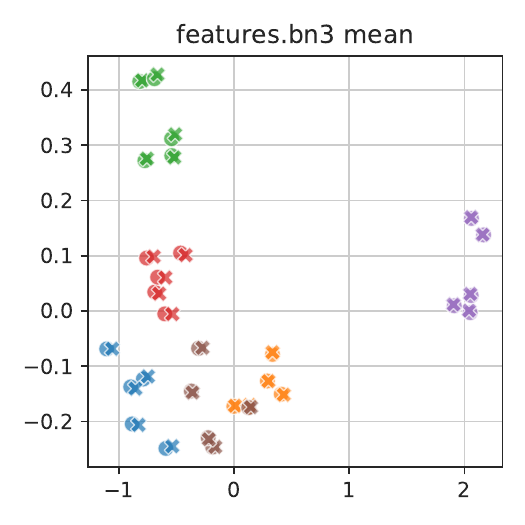} \hfil
        \includegraphics[width=0.36\textwidth]{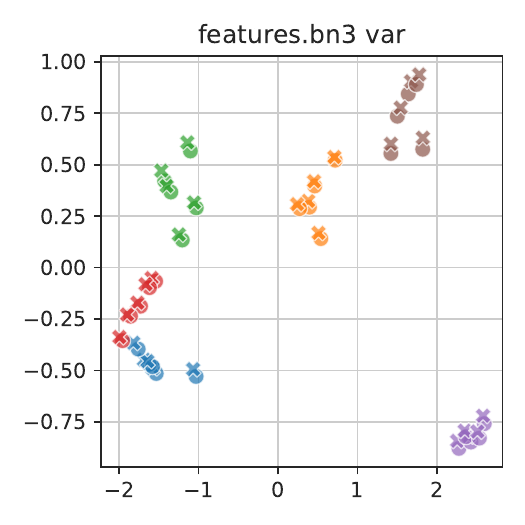}

        \includegraphics[width=0.36\textwidth]{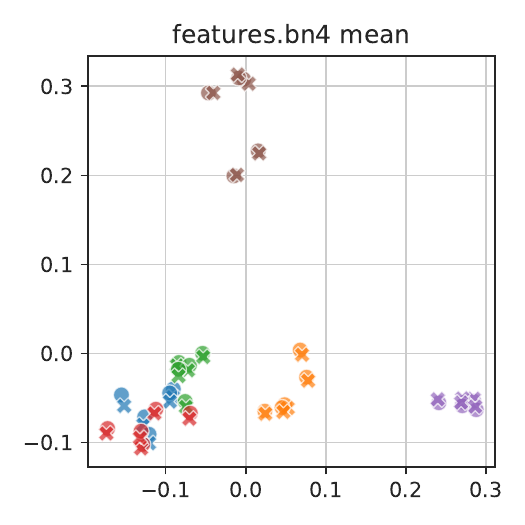} \hfil
        \includegraphics[width=0.36\textwidth]{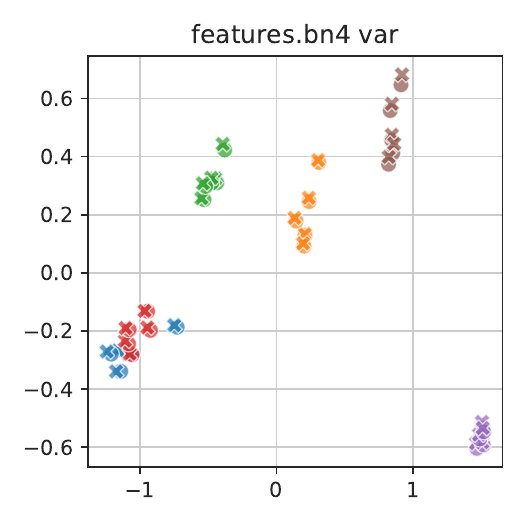}

        \includegraphics[width=0.36\textwidth]{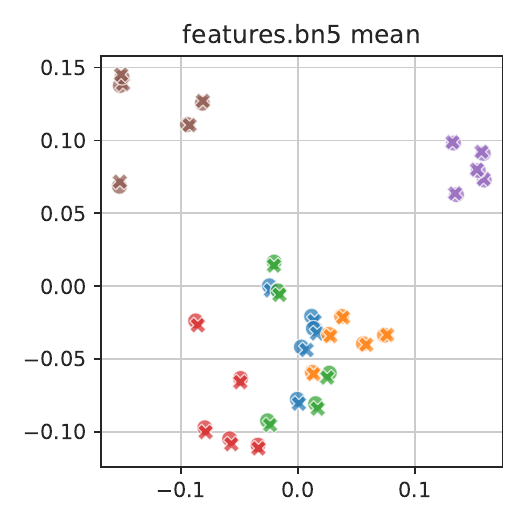} \hfil
        \includegraphics[width=0.36\textwidth]{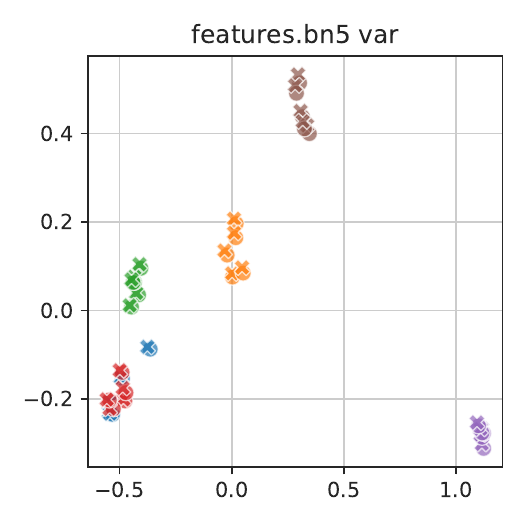}

        \includegraphics[width=0.36\textwidth]{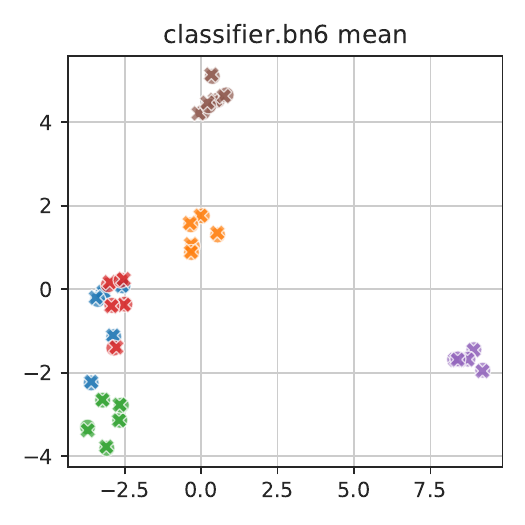} \hfil
        \includegraphics[width=0.36\textwidth]{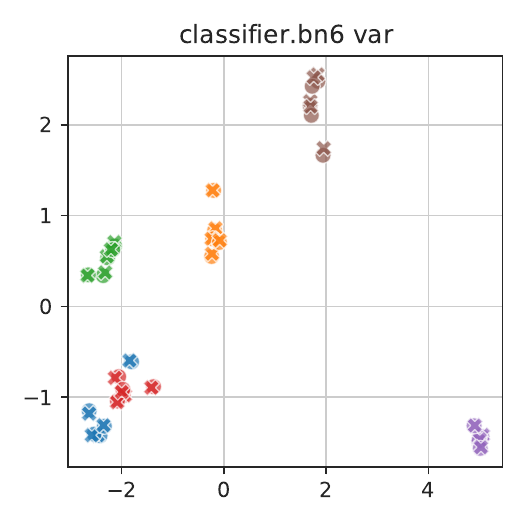}

        \includegraphics[width=0.36\textwidth]{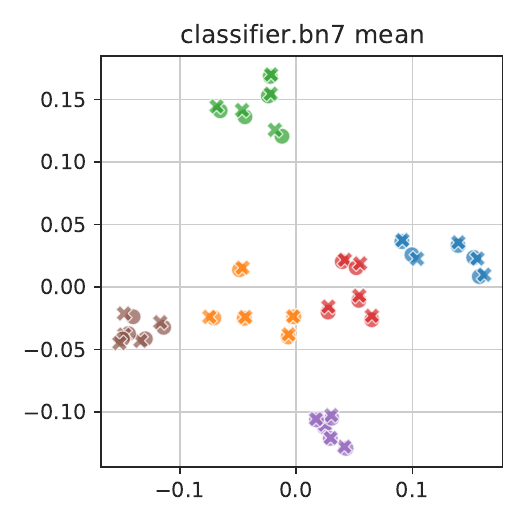} \hfil
        \includegraphics[width=0.36\textwidth]{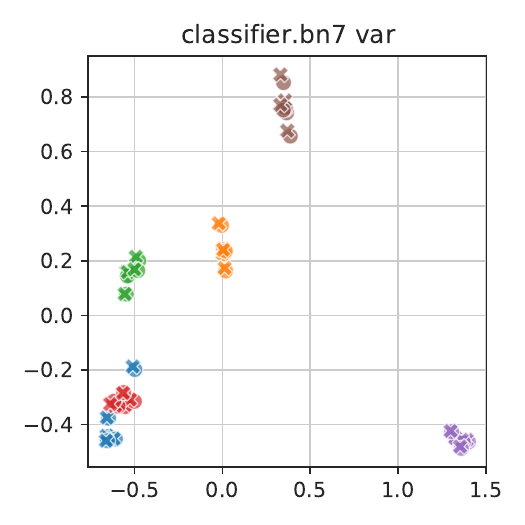}
        \caption{DomainNet dataset}
    \end{subfigure}
    \caption{Visualization of users' BN statistics by PCA. The BN after the first convolutional layer and the first linear layer is extracted. The model is trained by LBN + DBN on 100\% AT users on the Digits dataset.}
    \label{fig:bn_pca_ex}
\end{figure*}

\end{document}